\pdfoutput=1

\documentclass[twoside]{article}

\usepackage[accepted]{aistats2021}

\usepackage[round]{natbib}

\usepackage{url}   
\usepackage{bm}
\usepackage{amsmath}
\usepackage{amsthm}
\usepackage{amsfonts}       
\usepackage{algorithm}
\usepackage{algorithmic}
\usepackage{subcaption}
\usepackage{graphicx}
\graphicspath{ {img/} }

\usepackage{hyperref}

\newtheorem{lemma}{Lemma}
\newtheorem{remark}{Remark}

\DeclareMathOperator*{\argmin}{arg\,min}

\newcommand{\RR}{\mathbb{R}}
\newcommand{\NN}{\mathbb{N}}

\newcommand{\cA}{{\mathcal A}}

\newcommand{\cH}{{\mathcal H}}

\newcommand{\cT}{{\mathcal T}}

\newcommand{\cV}{{\mathcal V}}

\newcommand{\cY}{{\mathcal Y}}
\newcommand{\cZ}{{\mathcal Z}}

\newcommand{\veps}{\varepsilon}

\begin{document}

%

%

\twocolumn[

\aistatstitle{Parametric Programming Approach for  More Powerful and General Lasso Selective Inference}

\aistatsauthor{ Vo Nguyen Le Duy \And Ichiro Takeuchi}

\aistatsaddress{ 
Nagoya Institute of Technology and RIKEN \\ duy.mllab.nit@gmail.com \And  
Nagoya Institute of Technology and RIKEN \\ takeuchi.ichiro@nitech.ac.jp} 
]

\begin{abstract}
Selective Inference (SI) has been actively studied in the past few years for conducting inference on the features of linear models that are adaptively selected by feature selection methods such as Lasso. 
The basic idea of SI is to make inference conditional on the selection event.
Unfortunately, the main limitation of the original SI approach for Lasso is that the inference is conducted not only conditional on the selected features but also on their signs --- this leads to loss of power because of over-conditioning.
Although this limitation can be circumvented by considering the union of such selection events for all possible combinations of signs, this is only feasible when the number of selected features is sufficiently small. 
To address this computational bottleneck, we propose a parametric programming-based method that can conduct SI without conditioning on signs even when we have thousands of active features. 
The main idea is to compute the continuum path of Lasso solutions in the direction of the selected test statistic, and identify the subset of the data space corresponding to the feature selection event by following the solution path.
The proposed parametric programming-based method not only avoids the aforementioned computational bottleneck but also improves the performance and practicality of SI for Lasso in various respects. 
We conduct several experiments to demonstrate the effectiveness and efficiency of our proposed method.
\end{abstract}

 \section{Introduction}

Reliable machine learning (ML), which is the problem of assessing the reliability of data-driven knowledge obtained by ML algorithms, is one of the most important issues in the ML community.
Among various approaches for reliable ML, \emph{selective inference (SI)} 
has been recognized as a new promising approach for assessing the statistical reliability of data-driven hypotheses selected by complex data analysis algorithms.

SI was first introduced as a statistical inference tool for the features selected by Lasso~\citep{tibshirani1996regression}. 
Although various properties of Lasso have been extensively studied in the past decades (see, e.g., \cite{hastie2015statistical}), exact statistical inference such as computing $p$-values or confidence intervals for \emph{adaptively} selected features by Lasso has only recently begun to be actively studied in the context of SI~\citep{lee2016exact, fithian2014optimal, liu2018more}. 

The main idea of SI is to make inference for the selected features \emph{conditional on the selection event}, leading to exact \emph{valid} inference on adaptively selected features by Lasso is possible in the sense that $p$-values for proper false positive rate control or confidence intervals with proper coverage guarantees can be obtained. 
After the seminal work \citep{lee2016exact}, \emph{conditional inference}-based SI has been actively studied and applied to various problems~\citep{bachoc2014valid, fithian2014optimal, fithian2015selective, choi2017selecting, tian2018selective, chen2019valid, hyun2018post, bachoc2018post, charkhi2018asymptotic, loftus2014significance, loftus2015selective, panigrahi2016bayesian, tibshirani2016exact, yang2016selective, suzumura2017selective, tanizaki2020computing, duy2020computing, duy2020quantifying, sugiyama2020more}.

\textbf{Existing works and their drawbacks.}
Let $\cA$ be a random variable indicating the set of the selected features by applying Lasso on any random data sample and $\bm s$ be their signs.
Then in the seminal work 
\citep{lee2016exact}, the authors showed that the selection event $\{\cA = \cA_{\rm obs}, \bm s = {\bm s}_{\rm obs}\}$ is characterized as a polytope in the data space, where $\cA_{\rm obs}$ and ${\bm s}_{\rm obs}$ are the corresponding observations (see \S2 for detailed setup), leading to the sampling distribution of the test-statistic in the form of a \emph{truncated Normal distribution}.
However, it is well-known that conditioning on the signs leads to low statistical power because of \emph{over-conditioning},
which is widely recognized as a major drawback of the current Lasso SI approach.

%
\begin{figure*}[!t]
\centering
\includegraphics[width=.7\linewidth]{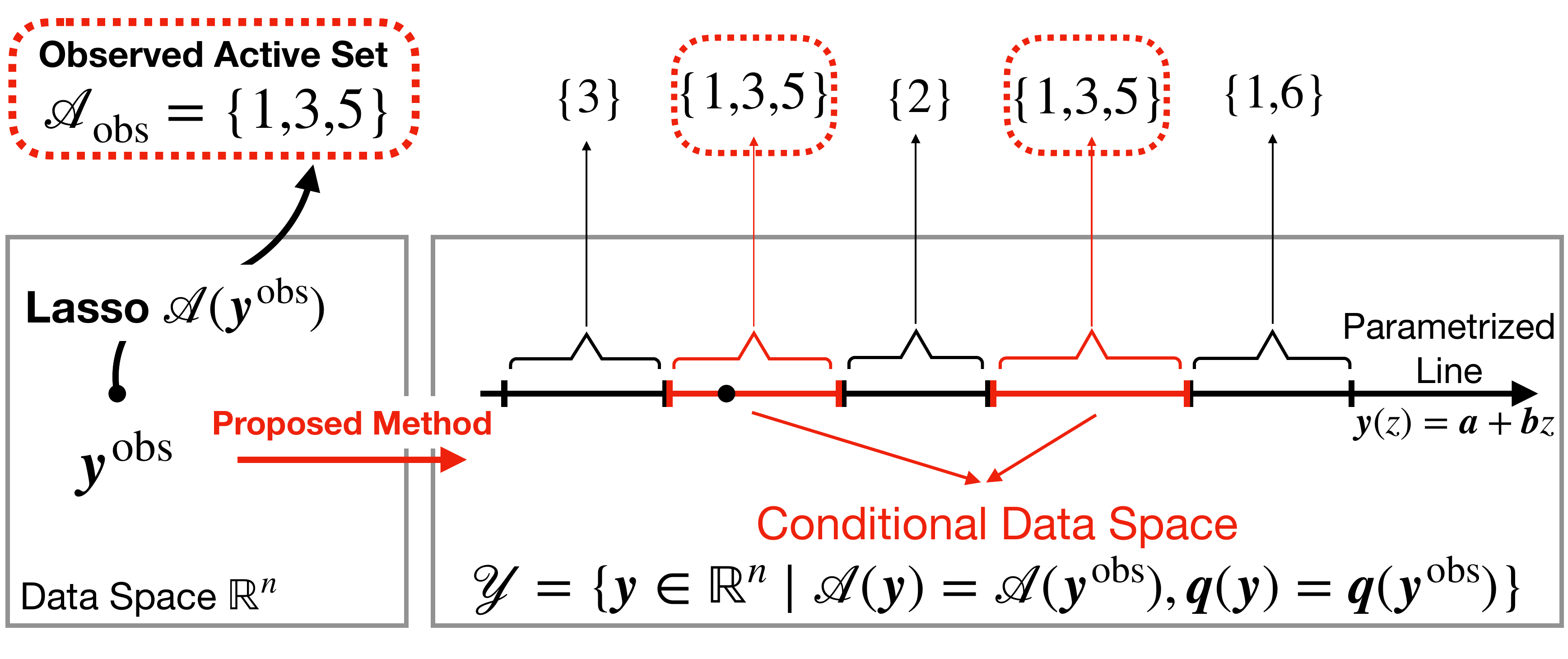}  
\caption{
Schematic illustration of the proposed method.
By applying Lasso on the observed data $\bm y^{\rm obs}$, we obtain the observed active set $\cA_{\rm obs}$.
The statistical inference for each selected feature is conducted conditional on the subspace $\cY$ whose data has the same active set as $\bm y^{\rm obs}$.
We introduce a parametric programing method for characterizing the conditional data space $\cY$ by searching on the parametrized line.
}
\label{fig:fig_intro}
\end{figure*}
\cite{lee2016exact} also discussed the solution to overcome the drawback by conducting conditional inferences without sign event $\{\cA = \cA_{\rm obs}\}$, which can be characterized by $2^{|\cA_{\rm obs}|}$ polytopes.
If the number of selected features 
$|\cA_{\rm obs}|$ is moderate (e.g., up to 15), it is feasible to consider 
all affine constraints of all these 
$2^{|\cA_{\rm obs}|}$ polytopes. 
However, if $|\cA_{\rm obs}|$ is large, it is \emph{infeasible} to enumerate the whole affine constraints for exponentially increasing number of polytopes. 
%
%

Recently, \cite{liu2018more} have proposed two approaches, in which the problem settings are different from \cite{lee2016exact}, to improve the power.
However, in their first approach, it is only applicable when the number of features $p$ is smaller than the number of instances $n$. 
%
%
%
%
In the second approach, they still consider an exponentially large number of all possible sign vectors.
%
This paper is motivated by Section 6 of \cite{liu2018more} in which they provide a recipe for constructing more powerful conditional SI methods.
In the other direction, \cite{tian2018selective} and \cite{terada2019selective} proposed methods using randomization. 
A drawback of these randomization-based approaches 
including simple data-splitting approach 
is that further randomness is added in both feature selection and inference stages.

Many machine learning tasks involve careful tuning of a \emph{regularization parameter} $\lambda$ that controls the balance between an empirical loss term and a regularization term, e.g., 
commonly by 
cross-validation (CV).
However, most of the current Lasso SI methods assume a pre-specified $\lambda$ and ignore the fact that $\lambda$ is selected based on the data  because the selection event of cross-validation i
is difficult to characterize.
%
%
%
\cite{loftus2015selective} and \cite{markovic2017unifying} proposed solutions to incorporate CV event.
However, the former requires additional conditioning on all intermediate models which leads to loss of power and the latter considers a randomization version of CV instead of the vanilla CV.
%
%

\paragraph{Contribution.} Our contributions are as follows:

$\bullet$ In this paper, we propose a new SI approach based on \emph{parametric-programming (PP)} \citep{Ritter84, Allgower93, Gal95, Best96}, which we call \emph{PP-based SI}, for resolving several major limitations of the seminal polytope-based SI proposed by  \cite{lee2016exact}.
The polytope-based SI is applicable when the selection event can be characterized as a polytope in the data space.
Otherwise, the only way is to consider extra conditions, e.g., sign conditioning, so that the over-conditioned event is characterized as a polytope, which leads to loss of statistical power.
In contrast, with the proposed PP-based SI, it is possible to characterize the selection event even if they cannot be described by a polytope.

$\bullet$ We introduce a method to compute the continuum path of Lasso solutions in the direction of interest, which is sub-sequently used to identify the exact sampling distribution of the test statistic with the minimum amount of conditioning.
Therefore, the PP-based SI can fundamentally resolve the over-conditioning problem, which was a major concern in polytope-based SI, to achieve the high statistical power.
Although the concept of PP has been used in various 
problems \citep{osborne00a, Efron04a, HasRosTibZhu04, Rosset05, BacHecHor06, RosZhu07, Tsuda07, Lee07, garrigues2008homotopy, Takeuchi09a, Karasuyama11, hocking11a, Karasuyama12a, lei2019fast}, this is the first study that introduces a piecewise-linear PP approach for characterizing the selection events in SI.

$\bullet$ Furthermore, by using PP-based SI, we can perform SI \emph{with minimal conditioning} for regularization parameter selection by cross-validation, which is complicated and was not possible with polytope-based SI.
Besides, we show that our proposed method is general and can be applied in several settings as well as other selection models such as elastic net and interaction model. 

%

%
%
Figure \ref{fig:fig_intro} shows the schematic illustration of the proposed method.
For reproducibility, our implementation is available at
\begin{center}
\href{https://github.com/vonguyenleduy/parametric_lasso_selective_inference}{https://github.com/vonguyenleduy/parametric\_lasso\_\\selective\_inference}
\end{center}



\section{Problem Statement}

To formulate the problem, we consider a random response vector 
\begin{equation} \label{eq:random_vector}
{\bm Y} = (Y_1, ..., Y_n)^\top \sim \NN({\bm \mu}, \Sigma),
\end{equation}
where $n$ is the number of instances, ${\bm \mu}$ is modeled as a linear function of $p$ features ${\bm x}_1, ..., {\bm x}_p \in \RR^n$, and $\Sigma \in \RR^{n \times n}$ is a covariance matrix which is known or estimable from independent data.
The goal is to statistically quantify the significance of the relation between the features and response while properly controlling the false positive rate. 
To achieve the goal, the authors in \cite{lee2016exact} have proposed a practical SI framework, in which a subset of features is first ``selected'' by the Lasso, and the inferences are then conducted for each selected feature.

\paragraph{Feature selection and its selection event.}
Given an observed response 
vector ${\bm y}^{\rm obs} \in \RR^n$ sampled from the model (\ref{eq:random_vector}), the Lasso optimization problem is given by
\begin{equation} \label{eq:lasso}
	\hat{{\bm \beta}} = \argmin \limits_{{\bm \beta} \in \RR^p} \frac{1}{2} \|{\bm y}^{\rm obs} - X {\bm \beta}\|^2_2 + \lambda \|{\bm \beta}\|_1,
\end{equation}
where $X \in \RR^{n \times p}$ is a feature matrix, and $\lambda \geq 0$ is a regularization parameter. 
Since the Lasso produces sparse solutions, the active set selected by applying the Lasso to ${\bm y}^{\rm obs}$ is defined as 
\begin{align}
	\cA_{\rm obs} = \cA({\bm y}^{\rm obs}) = \{ j : \hat{\beta}_j \neq 0\}.
\end{align}
Then, the event that the Lasso active set for a random vector $\bm Y$ is the same as ${\bm y}^{\rm obs}$ is written as
\begin{align}
	\left \{ \cA(\bm Y) = \cA({\bm y}^{\rm obs}) \right \}.
\end{align}

\paragraph{Statistical inference for the selected feature.}
The selected $j^{\rm th}$ coefficient is written as $\hat{\beta}_{j} = \bm \eta_j^\top \bm y^{\rm obs} $ by defining 
\begin{align} \label{eq:eta}
\bm{\eta}_j = X_{\cA_{\rm obs}} \left( X^\top_{\cA_{\rm obs}} X_{\cA_{\rm obs}}\right)^{-1} \bm{e}_j,
\end{align}
where $\bm{e}_j \in \RR^{|\cA_{\rm obs}|}$ is a basis vector with a $1$ at position $j^{\rm th}$.
For the inference on the $j^{\rm th}$ selected feature, we consider the following statistical test
\begin{equation}
 {\rm H}_{0, j}: \bm \eta_j^\top \bm \mu  = 0 \quad \text{vs.} \quad {\rm H}_{1, j}: \bm \eta_j^\top \bm \mu \neq 0.
\end{equation}
%
%
%
%
Since the hypothesis is generated from the data, selection bias exists.  
In order to correct the selection bias, we have to remove the information that has been used for initial hypothesis generating process. 
This is achieved by considering the sampling distribution of the test statistic $\bm{\eta}^\top_j \bm Y$ conditional on the selection event, i.e.,
\begin{equation}\label{eq:condition_model}
	\bm{\eta}^\top_j \bm Y \mid \left \{ \cA(\bm Y) = \cA({\bm y}^{\rm obs}), {\bm q}({\bm Y}) = {\bm q} ({\bm y}^{\rm obs})\right \},
\end{equation}
where ${\bm q}({\bm Y}) = (I_n - {\bm c} {\bm \eta}^\top_j) {\bm Y}$ with $\bm c = \Sigma {\bm \eta}_j ({\bm \eta}^\top_j \Sigma {\bm \eta}_j)^{-1}$.
The second condition ${\bm q}({\bm Y}) = {\bm q} ({\bm y}^{\rm obs})$ 
indicates the component that is independent of the test statistic for a random vector $\bm Y$ is the same as the one for $\bm y^{\rm obs}$.
The ${\bm q}(\bm Y)$ corresponds to the component $\bm z$ in the seminal paper (see \cite{lee2016exact}, Sec 5, Eq 5.2 and Theorem 5.2).

Once the selection event is identified, we can easily compute the pivotal quantity
\begin{equation}\label{eq:pivotal_quantity}
\footnotesize{
	F^{\cZ}_{\bm{\eta}^\top_j \bm \mu, {\bm \eta}^\top_j \Sigma {\bm \eta}_j} (\bm{\eta}^\top_j \bm Y) 
	\mid \left \{ \cA(\bm Y) = \cA({\bm y}^{\rm obs}), {\bm q}({\bm Y}) = {\bm q} ({\bm y}^{\rm obs}) \right \},}
\end{equation}
which is the c.d.f. of the truncated Normal distribution with mean $\bm{\eta}^\top_j \bm \mu$, variance ${\bm \eta}^\top_j \Sigma {\bm \eta}_j$, and the truncation region $\cZ$ which is calculated based on the selection event.
The pivotal quantity is crucial for calculating $p$-value and confidence interval.
Based on the pivotal quantity, we can consider \emph{selective type I error} or \emph{selective $p$-value} \citep{fithian2014optimal} in the form of 
 \begin{align} \label{eq:selective_p_value}
  P^{\rm selective}_j = 2\ \min\{\pi_j, 1 - \pi_j\},
\end{align}
where  $\pi_j = 1 - F^{\cZ}_{0, {\bm \eta}^\top_j \Sigma {\bm \eta}_j} (\bm{\eta}^\top_j {\bm Y})$,
which is \emph{valid} in the sense that 
\begin{align*}
	{\rm Prob}_{{\rm H}_{0, j}} \left(P^{\rm selective}_j < \alpha \right) = \alpha, \forall \alpha \in [0, 1].
\end{align*}
Furthermore,
to obtain $1 - \alpha$ confidence interval for any $\alpha \in [0, 1]$, by inverting the pivotal quantity in Equation (\ref{eq:pivotal_quantity}), we can find the smallest and largest values of $\bm{\eta}^\top_j \bm \mu$ such that the value of pivotal quantity remains in the interval $\left[ \frac{\alpha}{2}, 1- \frac{\alpha}{2} \right]$ \citep{lee2016exact}. 

However, the main challenge is that characterizing $\cA(\bm Y) = \cA({\bm y}^{\rm obs})$ in Equation (\ref{eq:condition_model}) is intractable because we have to consider $2^{\left |\cA({\bm y}^{\rm obs}) \right|}$ possible sign vectors. 
To overcome this issue, \cite{lee2016exact} consider inference conditional not only on the selected features but also on their signs.
Unfortunately, additionally considering the signs leads to low statistical power because of \emph{over-conditioning}.

In the next section, we will introduce a method for identifying the minimum amount of conditioning 
$\left \{ \cA(\bm Y) = \cA({\bm y}^{\rm obs}), {\bm q}({\bm Y}) = {\bm q} ({\bm y}^{\rm obs})\right \}$, which leads to high statistical power.
The main idea is to compute the path of Lasso solutions in the direction of interest $\bm \eta_j$.
By focusing on the line along $\bm \eta_j$, we can skip majority of the polytopes that do not affect the truncated Normal sampling distribution because they do not intersect with this line.
In other words, we can skip majority of combinations of signs that never appear when applying Lasso to the data on the line.

\section{Proposed Method}
In this section, we propose a parametric programming approach for characterizing conditioning event in (\ref{eq:condition_model}). 
The schematic illustration is shown in Figure \ref{fig:fig_intro}. 

\subsection{Conditional Data Space Characterization}
Let us define the set of $\bm y \in \RR^n$ which satisfies the conditions in Equation (\ref{eq:condition_model}) as 
\begin{equation} \label{eq:conditional_data_space}
	\hspace{-0.05mm} \cY = \{ {\bm y} \in \RR^{n} \mid \cA({\bm y}) = \cA({\bm y}^{\rm obs}), {\bm q} (\bm y) = {\bm q} ({\bm y}^{\rm obs})\}.
\end{equation}
According to the second condition, the data in $\cY$ is restricted to a line (see Sec 6 in \cite{liu2018more}, and \cite{fithian2014optimal}).
Therefore, the set $\cY$ can be re-written, using a scalar parameter $z \in \RR$, as
\begin{equation} \label{eq:parametrized_data_space}
	\cY = \left \{ {\bm y}(z) = {\bm a} + {\bm b} z \mid z \in \cZ \right \},
\end{equation}
where 
${\bm a} = {\bm q}(\bm y^{\rm obs})$, 
${\bm b} = \Sigma {\bm \eta}_j ({\bm \eta}^\top_j \Sigma {\bm \eta}_j)^{-1} $,
and 
\begin{equation} \label{eq:truncation_region_z}
	\cZ = \left \{ z \in \RR \mid \cA({\bm y}(z)) = \cA({\bm y}^{\rm obs}) \right \}.
\end{equation}
Now, let us consider a random variable $Z \in \RR$ and its observation $z^{\rm obs} \in \RR$, which satisfy ${\bm Y} = {\bm a} + {\bm b} Z$ and ${\bm y}^{\rm obs} =  {\bm a} + {\bm b} z^{\rm obs}$. 
The conditional inference in (\ref{eq:condition_model}) is re-written as the problem of characterizing the sampling distribution of 
\begin{equation} \label{eq:condition_parametric}
	Z \mid \left \{ Z \in \cZ \right \}.
\end{equation}
Since $Z \sim \NN(0, {\bm \eta}^\top_j \Sigma {\bm \eta}_j)$ under the null hypothesis, the law of $Z \mid Z \in \cZ$ follows a truncated Normal distribution.
Once the truncation region $\cZ$ is identified, the pivotal quantity in Equation (\ref{eq:pivotal_quantity}) is equal to $F^{\cZ}_{0, {\bm \eta}^\top_j \Sigma {\bm \eta}_j} (Z)$, and can be easily obtained.
Thus, the remaining task is to characterize $\cZ$.
\paragraph{Characterization of truncation region $\cZ$.}
Let us introduce the optimization problem (\ref{eq:lasso}) with parametrized response vector ${\bm y}(z)$ for $z \in \RR$ as
\begin{equation} \label{eq:parametric_lasso}
	\hat{{\bm \beta}}(z) = \argmin \limits_{{\bm \beta} \in \RR^p} \frac{1}{2} \|{\bm y}(z) - X {\bm \beta}\|^2_2 + \lambda \|{\bm \beta}\|_1.
\end{equation}
The subdifferential of the $\ell_1$-norm at $\hat{{\bm \beta}}(z)$ is defined as follows:
\begin{align*}
	\partial \|\hat{{\bm \beta}}(z)\|_1 = {\bm s}(z) : 
	\begin{cases}
	s_j (z) = {\rm sign}(\hat{\beta}_j(z)) & \text{ if } \hat{\beta}_j(z) \neq 0\\
	 s_j (z) \in [-1, 1]  & \text{ if } \hat{\beta}_j(z) = 0
	\end{cases},
\end{align*}
where we denote ${\bm s}(z) = {\rm sign}(\hat{{\bm \beta}} (z))$.
Then, for any $z$ in $\RR$, the optimality condition is given by
\begin{align}
	X^\top\left ( X \hat{{\bm \beta}}(z) - {\bm y}(z) \right ) + \lambda {\bm s}(z) = 0, 
\end{align}
${\bm s}(z) \in \partial \|\hat{{\bm \beta}}(z)\|_1$. To construct the truncation region $\cZ$ in Equation (\ref{eq:truncation_region_z}), we have to 1) compute the entire path of $\hat{{\bm \beta}}(z)$, and 2) identify the set of intervals of $z$ on which $\cA({\bm y}(z)) = \cA({\bm y}^{\rm obs})$.
However, it seems intractable to compute $\hat{{\bm \beta}}(z)$ for infinitely many values of $z \in \RR$.
Our main idea to overcome this difficulty is to propose a parametric programming method for efficiently computing a finite number of ``transition points'' at which the active set changes.

\subsection{A Piecewise Linear Homotopy}
We now derive the main technique.
We show that $\hat{{\bm \beta}}(z)$ is a piecewise linear function of $z$.
To make the notation lighter, we write $\cA_z = \cA({\bm y}(z))$, and we denote \emph{the set of inactive features as $\cA^c_z$.}

\begin{lemma} \label{lemma:piecewise_linear}
Consider two real values $z^\prime$ and $z$ $(z^\prime > z)$. 
Suppose $|s_j(z)| < 1$ for all $j \in \cA^c_z$, $|s_j(z^\prime)| < 1$ for all $j \in \cA^c_{z^\prime}$, and $X^\top_{\cA_z} X_{\cA_z}$ is invertible.
If $\hat{\bm \beta}_{\cA_z}(z)$  and $\hat{\bm \beta}_{\cA_{z^\prime}}(z^\prime)$ have the same active set and the same signs, then we have
\begin{align} 
	\hat{\bm \beta}_{\cA_z}(z^\prime) - \hat{\bm \beta}_{\cA_z}(z) 
	&= 
	{\bm \psi}_{\cA_z}(z) \times (z^\prime - z), \label{eq:lemma1_eq1}\\
	\lambda {\bm s}_{\cA^c_z}(z^\prime) - \lambda {\bm s}_{\cA^c_z}(z) 	&=
	{\bm \gamma}_{\cA^c_z}(z) \times (z^\prime - z), \label{eq:lemma1_eq2}
\end{align}
where 
${\bm \psi}_{\cA_z}(z) = (X^\top_{\cA_z} X_{\cA_z})^{-1} X^\top_{\cA_z} {\bm b}$, 
and 
${\bm \gamma}_{\cA^c_z}(z) = X^\top_{\cA^c_z} {\bm b} - X^\top_{\cA^c_z} X_{\cA_z} {\bm \psi}_{\cA_z}(z)$.
\end{lemma}

\begin{proof}
From the optimality conditions of the Lasso, we have 
\begin{align}
	\hspace{-5pt}X^\top_{\cA_z} X_{\cA_z} \hat{\bm \beta}_{\cA_z}(z) 
	- X^\top_{\cA_z} {\bm y}(z) + \lambda {\bm s}_{\cA_z}(z) = 0, 			
	\label{eq:proof_eq1} \\
	\hspace{-5pt}X^\top_{\cA_{z^\prime}} X_{\cA_{z^\prime}} \hat{\bm \beta}_{\cA_{z^\prime}}(z^\prime) 
	- X^\top_{\cA_{z^\prime}} {\bm y}(z^\prime) + \lambda {\bm s}_{\cA_{z^\prime}}(z^\prime) = 0.
	\label{eq:proof_eq2} 
\end{align}
Then, by subtracting (\ref{eq:proof_eq1}) from (\ref{eq:proof_eq2}) and $\cA_z = \cA_{z^\prime}$, we have
\begin{align*}
	\hat{\bm \beta}_{\cA_z}(z^\prime) - \hat{\bm \beta}_{\cA_z}(z) 
	&=
	(X^\top_{\cA_z} X_{\cA_z})^{-1} X^\top_{\cA_z} ({\bm y}(z^\prime) - {\bm y}(z))\\
	&=(X^\top_{\cA_z} X_{\cA_z})^{-1} X^\top_{\cA_z} ({\bm a} + {\bm b} z^\prime - {\bm a} - {\bm b} z)\\
	&=(X^\top_{\cA_z} X_{\cA_z})^{-1} X^\top_{\cA_z} {\bm b} \times (z^\prime - z).
\end{align*}
Thus, we achieve Equation (\ref{eq:lemma1_eq1}). Next, from the optimality conditions of the Lasso, we also have 
\begin{align}
	- X^\top_{\cA^c_z} X_{\cA_z} \hat{\bm \beta}_{\cA_z}(z) + X^\top_{\cA^c_z} {\bm y}(z) 
	&= 
	\lambda {\bm s}_{\cA^c_z}(z), \label{eq:proof_eq3}\\
	- X^\top_{\cA^c_{z^\prime}} X_{\cA_{z^\prime}} \hat{\bm \beta}_{\cA_{z^\prime}}(z^\prime) + X^\top_{\cA^c_{z^\prime}} {\bm y}(z^\prime) 
	&= 
	\lambda {\bm s}_{\cA^c_{z^\prime}}(z^\prime) \label{eq:proof_eq4}.
\end{align}
Similarly, by subtracting (\ref{eq:proof_eq3}) from (\ref{eq:proof_eq4}) and $\cA_z = \cA_{z^\prime}$, we can easily achieve Equation (\ref{eq:lemma1_eq2}).
\end{proof}

\begin{remark}
{\normalfont
 In this paper, we assume the uniqueness of the Lasso solution $\hat{\bm \beta}(z)$ for all $z \in \RR$ as well as $|s_j(z)| < 1$ for all $j \in \cA^c_z$ and the invertibility of $X^\top_{\cA_z} X_{\cA_z}$.
These assumptions are justified by assuming the columns of $X$ are in general position~\citep{tibshirani2013lasso}. 
Parametric programming methods for handling the rare cases where these assumptions are not satisfied have been studied, e.g., in \cite{Best96}, and can be applied to our problem setup.
In practice, when the design matrix is not in general position, it is also common to introduce an additional ridge penalty term, resulting in the elastic net \citep{zou2005regularization}.
Our proposed method can be extended to the elastic net case (see Appendix \ref{ext:elastic_net} for the details).  
 %
%
%
%
}
\end{remark}

\paragraph{Computation of the transition point.} From Lemma \ref{lemma:piecewise_linear}, the solution $\hat{\bm \beta}(z)$ is a linear function of $z$ until $z$ reaches a transition point at which either an element of $\hat{\bm \beta}(z)$ becomes zero or a component of ${\bm s}(z)$ becomes one in absolute value.
We now introduce how the transition point is identified.

\begin{lemma}\label{lemma:transition_point}
Let $z$ be a real value such that $\max_{j \in \cA^c_z} |s_j(z)| < 1$.
Then, 
$\cA_{z^\prime} = \cA_z$, 
$\max_{j \in \cA^c_{z^\prime}} |s_j(z^\prime)| < 1$, and 
${\bm s}(z) = {\bm s}(z^\prime)$ 
for any real value $z^\prime$ 
in the interval $[z, z + t_z)$, 
where $z + t_z$ is the value of transition point, 
\begin{gather}
	t_z = \min 
	\left \{ 
	t^1_{z},
	t^2_{z}
	\right \}, \\
	t^1_{z} = \min \limits_{j \in \cA_z} \left( - \frac{\hat{\beta}_j(z)}{\psi_j(z)} \right)_{++}, \label{eq:t_1_z} \\
	t^2_{z} = \min \limits_{j \in \cA^c_z} \left( \lambda \frac{{\rm sign}(\gamma_j(z)) - s_j(z)}{\gamma_j(z)} \right)_{++} \label{eq:t_2_z}.
\end{gather}
%
Here, we use the convention that for any $m \in \RR$, $(m)_{++} = m$ if $m > 0$, and $(m)_{++} = \infty$ otherwise.

\end{lemma}
\begin{proof}
From Equation (\ref{eq:lemma1_eq1}), we can see that $\hat{\bm \beta}_{\cA_z}(z)$ is a function of $z$.
For a real value $z$, there exists $t^1_{z}$ such that for any real value $z^\prime$ in $[z, z + t^1_{z})$, all elements of $\hat{\bm \beta}_{\cA_{z^\prime}}(z^\prime)$ remain the same signs with $\hat{\bm \beta}_{\cA_z}(z)$.
Similarly, from Equation (\ref{eq:lemma1_eq2}), we can see that $\bm s_{\cA^c_z}(z)$ is a function of $z$.
Then, for a real value $z$, there exists $t^2_{z}$ such that for any real value $z^\prime$ in $[z, z + t^2_{z})$, all elements of $\bm s_{\cA^c_{z^\prime}}(z^\prime)$ are smaller than 1 in absolute value.
Finally, by taking $t_z = \min \{t^1_{z}, t^2_{z}\}$, we obtain the interval in which the active set and signs of Lasso solution remain the same.
The remaining task is how to compute $t^1_{z}$ and $t^2_{z}$.
We defer the detailed derivations of $t^1_{z}$ and $t^2_{z}$ to the Appendix \ref{appendix:proof_lemma_2}.
\end{proof}

\begin{algorithm}[!t]
\renewcommand{\algorithmicrequire}{\textbf{Input:}}
\renewcommand{\algorithmicensure}{\textbf{Output:}}
\begin{footnotesize}
 \begin{algorithmic}[1]
  \REQUIRE $X, {\bm y}^{\rm obs}, \lambda, [z_{\rm min}, z_{\rm max}]$
	\vspace{2pt}
	\STATE Compute Lasso solution and obtain observed $\cA_{\rm obs}$ for data $(X, {\bm y}^{\rm obs})$
	\vspace{2pt}
	\FOR {each selected feature $j \in \cA_{\rm obs}$}
		\vspace{2pt}
		\STATE Compute $\bm \eta_j$ $\leftarrow$ Equation (\ref{eq:eta}) 
		\vspace{4pt}
		\STATE Compute $\bm a$ and $\bm b$ $\leftarrow$ Equation (\ref{eq:parametrized_data_space})
		\vspace{4pt}
		\STATE $ \hat{\bm \beta}(z), \cA_z \leftarrow {\tt compute\_solution\_path}$ ($X$, $\lambda$,  $\bm a$, $\bm b$, $[z_{\rm min}, z_{\rm max}]$)
		\vspace{4pt}
		\STATE Truncation region $\cZ \leftarrow \{z: \cA_z = \cA_{\rm obs}\}$
		\vspace{4pt}
		\STATE $P^{\rm selective}_j \leftarrow $ Equation (\ref{eq:selective_p_value}) (and$/$or selective confidence interval of $\beta_j$)
		\vspace{2pt}
	\ENDFOR
	\vspace{2pt}
  \ENSURE $\{P^{\rm selective}_j\}_{j \in \cA_{\rm obs}}$  (and$/$or selective confidence intervals of $\beta_j, j \in \cA_{\rm obs}$)
 \end{algorithmic}
\end{footnotesize}
\caption{{\tt parametric\_lasso\_SI}}
\label{alg:parametric_lasso_SI}
\end{algorithm}

\subsection{Algorithm}

In this section, we show the detailed algorithm of our proposed parametric programming method. In Algorithm \ref{alg:parametric_lasso_SI}, for feature selection step, we just simply apply Lasso to the data $(X, \bm y^{\rm obs})$, and obtain the active set $\cA_{\rm obs}$.
Then, we conduct SI for each selected feature.
For testing $\beta_j, j \in \cA_{\rm obs},$ we first obtain the direction of interest $\bm \eta_j$, which can be easily computed as in Equation (\ref{eq:eta}). 
Second, the main task is to compute the solution path of $\hat{\bm {\beta}}(z) $ in Equation (\ref{eq:parametric_lasso}) for the parametrized response vector $\bm y (z)$,
where, note that, the parametrized solution $\hat{\bm {\beta}}(z)$ are different among different $j \in \cA_{\rm obs}$ since the direction of interest $\bm \eta_j$ depends on $j$.
This task can be done by Algorithm \ref{alg:solution_path}.
Finally, after having the path, we can easily obtain truncation region $\cZ$ which is used to compute selective $p$-value or selective confidence interval.

\begin{algorithm}[!t]
\renewcommand{\algorithmicrequire}{\textbf{Input:}}
\renewcommand{\algorithmicensure}{\textbf{Output:}}
\begin{footnotesize}
 \begin{algorithmic}[1]
  \REQUIRE $X, \lambda, \bm a, \bm b, [z_{\rm min}, z_{\rm max}]$
	\vspace{2pt}
	\STATE Initialization: $k = 0$, $z_k=z_{\rm min}$, $\cT = {z_k}$
	\vspace{2pt}
	\WHILE {$z_k < z_{\rm max}$}
		\vspace{2pt}
		\STATE $\bm y(z_k) = \bm a + \bm b z_k$
		\vspace{4pt}
		\STATE $t_{z_k}, \hat{\bm \beta}(z_k), \cA_{z_k} \leftarrow {\tt compute\_step\_size}(X, \bm y(z_k)$, $\lambda$)
		\vspace{4pt}
		\STATE $z_{k+1} = z_k + t_{z_k}$, $\cT = \cT \cup \{z_{k+1}\}$ \\ \vspace{2pt} ($z_{k+1}$ is the value of the next transition point)
		\vspace{2pt}
		\STATE $k = k + 1$
		\vspace{2pt}
	\ENDWHILE
	\vspace{2pt}
  \ENSURE $\{\hat{\bm \beta}(z_k)\}_{z_k \in \cT}, \{\cA_{z_k}\}_{z_k \in \cT}$
 \end{algorithmic}
\end{footnotesize}
\caption{{\tt compute\_solution\_path}}
\label{alg:solution_path}
\end{algorithm}

\begin{algorithm}[!t]
\renewcommand{\algorithmicrequire}{\textbf{Input:}}
\renewcommand{\algorithmicensure}{\textbf{Output:}}
\begin{footnotesize}
 \begin{algorithmic}[1]
  \REQUIRE $X, \bm y(z), \lambda$
	\STATE Compute primal$/$dual Lasso solution $\hat{\bm \beta}(z), \hat{\bm s}(z)$ for data $(X, \bm y(z))$
	\vspace{2pt}
	\STATE Obtain active set $\cA_z = \{j : \hat{\beta}_j(z) \neq 0\}$
	\vspace{4pt}
	\STATE Compute ${\bm \psi}_{\cA_z}(z)$, $\ {\bm \gamma}_{\cA^c_z}(z)$ $\leftarrow$ Lemma \ref{lemma:piecewise_linear}
	\vspace{4pt}
	\STATE $t^1_z$, $t^2_z$ $\leftarrow$ Equations (\ref{eq:t_1_z}) and (\ref{eq:t_2_z}) in Lemma \ref{lemma:transition_point}
	\vspace{4pt}
	\STATE $t_z = \min \{t^1_z, t^2_z\}$
	\vspace{2pt}
  \ENSURE $t_z, \hat{\bm \beta}(z), \cA_{z}$
 \end{algorithmic}
\end{footnotesize}
\caption{{\tt compute\_step\_size}}
\label{alg:compute_step_size}
\end{algorithm}

In Algorithm \ref{alg:solution_path}, a sequence of transition points are computed one by one.
The algorithm is initialized at $z_k = z_{\rm min}, k = 0$.
At each $z_k$, the task is to find the next transition point $z_{k+1}$, where the active set changes.
This task can be done by computing the step size in Algorithm \ref{alg:compute_step_size}.
This step is repeated until $z_k > z_{\rm max}$. 
The algorithm returns the sequences of Lasso solutions and transition points.

\paragraph{Choice of $[z_{\rm min}, z_{\rm max}]$.} 
Under the normality, very positive and negative values of $z$ does not affect the inference. Therefore, it is reasonable to consider range of values, e.g., $[ - 20 \sigma, 20 \sigma]$ \citep{liu2018more}, where $\sigma$ is the standard deviation of the sampling distribution of test statistic.


\subsection{Characterization of CV-based Tuning Parameter Selection Event}
In this section, we introduce a new way to characterize the \emph{minimal} selection event that $\lambda$ is chosen based on the data, e.g., via cross-validation, which is complicated and thus none of the currently available Lasso SI methods can handle. 
Given a set of regularization parameter candidates $\Lambda$, we denote $\cV(\bm y^{\rm obs}) =  \lambda^{\rm obs} \in \Lambda$ is the event that $\lambda^{\rm obs}$ is selected when performing validation on $\bm y^{\rm obs}$.
The conditional inference on selected feature $j$ when applying Lasso on $\{X, \bm y^{\rm obs}\}$ is then defined as 
\begin{align}
	\bm \eta_j^\top \bm Y \mid 
	 \{ 
		&\cA(\bm Y) = \cA(\bm y^{\rm obs}),  \nonumber \\ 
		&\cV(\bm Y) = \cV(\bm y^{\rm obs}),
		\bm q(\bm Y) = \bm q(\bm y^{\rm obs})
	 \}.
\end{align}
The conditional data space in (\ref{eq:parametrized_data_space}) with validation selection event is re-defined as 
\begin{align}
	\cY =\{\bm y(z) = \bm a + \bm bz \mid z \in \cZ_{\rm CV}\},
\end{align}
where 
$
	\cZ_{\rm CV} = \{z \in \RR \mid 
		\cA(\bm y(z)) = \cA(\bm y^{\rm obs}),
		\cV(\bm y(z)) = \cV(\bm y^{\rm obs})\}.
$
We now can easily construct $\cZ_1 = \{z \in \RR \mid \cA(\bm y(z)) = \cA(\bm y^{\rm obs})\}$ by using the proposed method in previous parts. The remaining task is to identify 
$
\cZ_2 = \{z \in \RR \mid \cV(\bm y(z)) = \cV(\bm y^{\rm obs})\}.
$ 
Finally, $\cZ_{\rm CV} = \cZ_1 \cap \cZ_2$.

For notational simplicity, we consider the case where the data is divided into training and validation sets, and the latter is used for selecting $\lambda$. The following discussion can be easily extended to cross-validation scenario. 
Let us re-write $
\{X, \bm{y}^{\rm obs}\} = 
\left \{
(X_{\rm train}\ X_{\rm val} )^\top \in \RR^{n \times p},
(\bm y^{\rm obs}_{\rm train}\ \bm y^{\rm obs}_{\rm val})^\top \in\RR^{n}
\right \}.
$ 
For $\lambda \in \Lambda$, the Lasso problem on parametrized training response vector is written as 
\begin{align*}
	\hat{\bm \beta}_\lambda(z) \in \argmin \limits_{\bm \beta \in \RR^p} \frac{1}{2} \|\bm y_{\rm train}(z) - X_{\rm train} \bm \beta\|^2_2 + \lambda \|\bm \beta\|_1.
\end{align*}
The validation error is defined as 
$
	E_\lambda(z) = \frac{1}{2} \|\bm y_{\rm val}(z) - X_{\rm val} \bm \hat{\bm \beta}_\lambda(z)\|^2_2.
$
Then, we can re-defined
$
	\cZ_2 = \{z \in \RR \mid E_{\lambda^{\rm obs}}(z) \leq E_\lambda(z) \text{ for any } \lambda \in \Lambda \}.
$
Since $\hat{\bm \beta}_\lambda(z)$ is a piecewise-linear function of $z$ and $\bm y_{\rm val}(z)$ is a linear function of $z$, the validation error $E_\lambda(z)$ is a picecewise-quadratic function of $z$. 
Now, for each $\lambda \in \Lambda$, we have a corresponding picecewise-quadratic function of $z$. 
Finally, we can identify $\cZ_2$ by finding the intervals of $z$ in which the validation error $E_{\lambda^{\rm obs}}(z)$ corresponding to $\lambda^{\rm obs}$ is minimum among a set of picecewise-quadratic functions.

\subsection{The Generality of the Proposed Method}

Since we can efficiently compute the path of Lasso solutions, our proposed method is flexible and can be easily extended to various respects. 
In \cite{liu2018more}, the main limitations are their method can not be applied when $p > n$, or requires huge computation time. 
With our method, all these limitations are resolved. 
We provide detailed discussions and solutions in Appendices \ref{ext:full_target} and \ref{ext:partial_target}.
Besides, we also apply the proposed method to other respects, which can not be solved by the methods in \cite{lee2016exact} and \cite{liu2018more}, including characterizing the minimum amount of conditioning in elastic net \citep{zou2005regularization} (Appendix \ref{ext:elastic_net}), marginal model (Appendix \ref{ext:marginal_model}), and interaction model (Appendix \ref{ext:interaction_model}).

\section{Experiment}
In this section, we will demonstrate the performance of the proposed method. 
Here, we present the main results.
Several additional experiments can be found in Appendix \ref{appendix:exp_details}.
%
%

\subsection{Experimental Setup}
We executed the code on Intel(R) Xeon(R) CPU E5-2687W v4 @ 3.00GHz.
\paragraph{Methods for comparison.} We show the false positive rates (FPRs), true positive rates (TPRs) and confidence intervals (CIs) for the following cases of conditional inferences:

$\bullet$ \textbf{TN-A}: conditional inference \emph{without} sign conditioning, which is mainly focused in this paper, 
\begin{align*}
\bm \eta_j^\top \bm {\bm Y} \mid \left \{\cA(\bm Y) = \cA_{\rm obs}, \bm q(\bm Y) = \bm q({\bm y}^{\rm obs}) \right \}.
\end{align*}
$\bullet$ \textbf{TN-As}: conditional inference with additional sign conditioning, which is mainly focused in \cite{lee2016exact},
\begin{align*}
\bm \eta_j^\top \bm {\bm Y} \mid \left \{ \cA(\bm Y) = \cA_{\rm obs}, \bm s = {\bm s}_{\rm obs}, \bm q(\bm Y) = \bm q({\bm y}^{\rm obs}) \right \},
\end{align*}
where $\bm s$ is the sign vector of Lasso solutions on $\bm Y$, and ${\bm s}_{\rm obs}$ is the sign vector of the Lasso solutions on ${\bm y}^{\rm obs}$.

We also show the FPRs, TPRs and CIs of data splitting (\textbf{DS}) method \citep{cox1975note}, which is the commonly used procedure for the purpose of selection bias correction. 
In this approach, the data is randomly divided in two halves — one half is used for model selection and the other half is used for inference.

\paragraph{Synthetic data generation.} We generated $n$ outcomes as $y_i = \bm x_i^\top \bm \beta + \veps_i$,  $i = 1, ..., n$, where $\bm x_i \sim \NN(0, I_p)$ in which $p = 5$, and $\veps_i \sim \NN(0, 1)$.
Here, we assume that the variance of the noise is known.
In practice, the variance can be estimated from \emph{independent} data.
We set the regularization parameter $\lambda = 1$ and significance level $\alpha = 0.05$.
We used Bonferroni correction to account for the multiplicity in all the experiments.
If we test $m$ selected features (hypotheses) at the same time, then the Bonferroni correction would test each individual hypothesis at $\alpha^\ast = \alpha / m$.
For the FPR experiments, all elements of $\bm \beta$ were set to 0 and we set $n \in \{100, 200, 300, 400, 500\} $.
For the TPR experiments, the first two elements of $\bm \beta$ were set to 0.25. 
We ran 100 trials for each $n \in \{50, 100, 150, 200\}$, and we repeated this experiments 10 times.
For the experiments of CIs, we set $n=100, p=10$, and the first 5 elements of $\bm \beta$ were set to 0.25.

\paragraph{Definition of TPR.} In SI, we only conduct statistical testing when there is at least one hypothesis discovered by the algorithm.
Therefore, the definition of TPR, which can be also called \emph{conditional power}, is as follows:
\begin{equation*}
	{\rm TPR} = \frac{{\rm \#\ correctly~detected\ \&\ rejected}}{{\rm \#\ correctly~detected}},
\end{equation*}
where ${\rm \#\ correctly~detected}$ is the number of truly positive features selected by the algorithm (e.g., Lasso) and  ${\rm \#\ rejected}$ is the number of truly positive features whose null hypothesis is rejected by SI.

\subsection{Numerical Results}
%
%
%

\begin{figure}[t]
\begin{subfigure}{.49\linewidth}
  \centering
  \includegraphics[width=\linewidth]{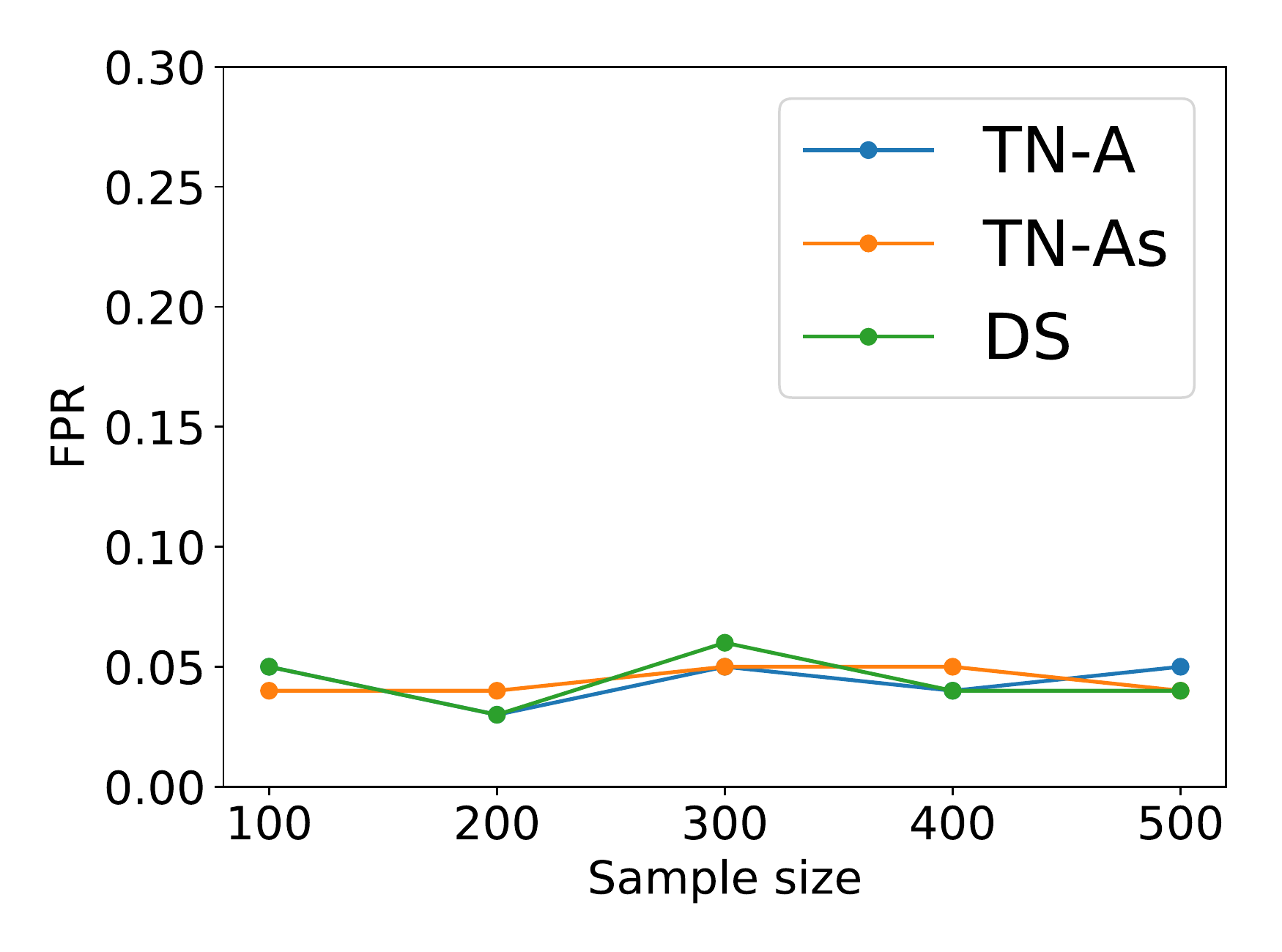}  
  \caption{FPR}
\end{subfigure}
\begin{subfigure}{.49\linewidth}
  \centering
  \includegraphics[width=\linewidth]{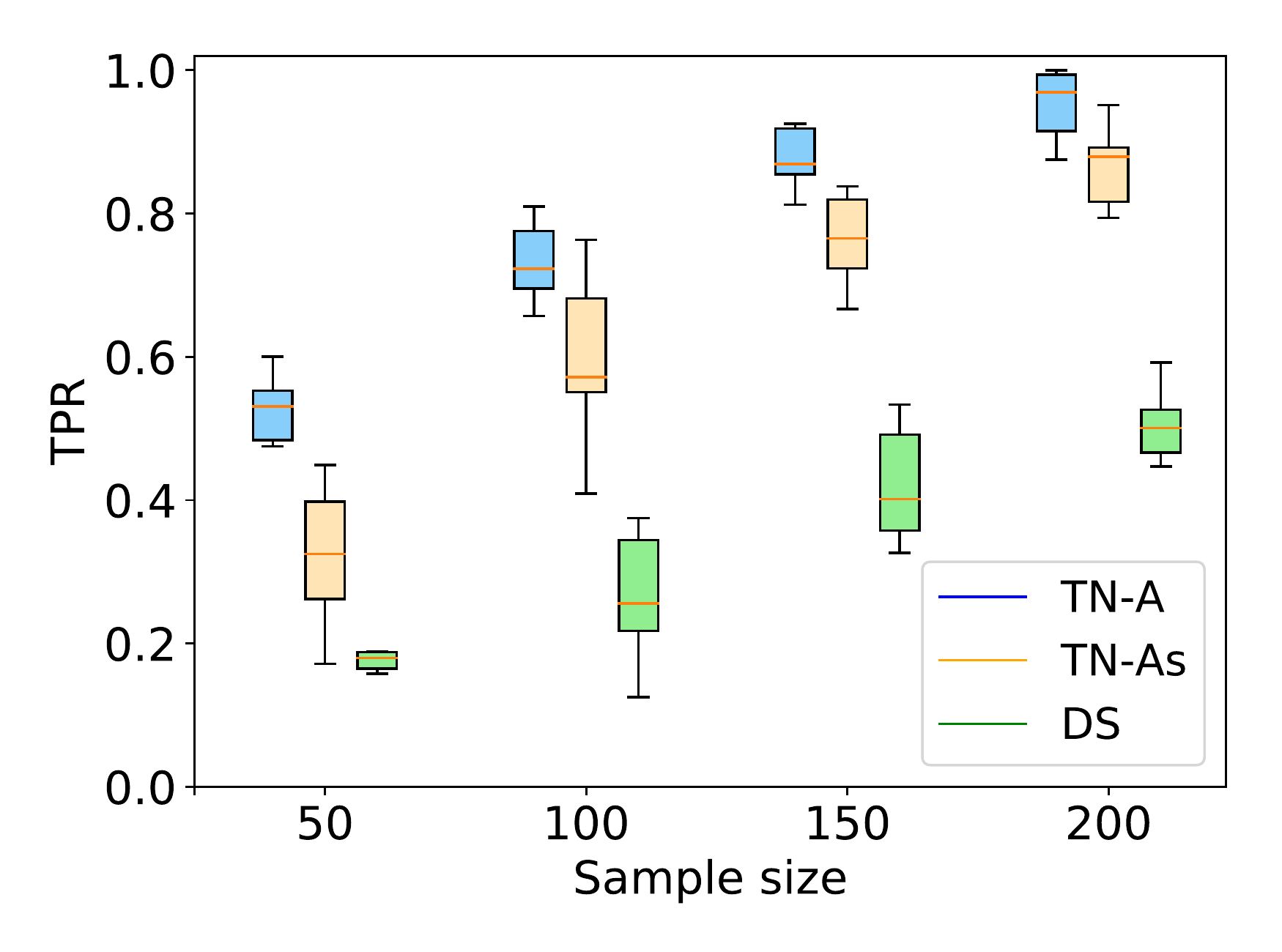} 
  \caption{TPR}
\end{subfigure}
\begin{subfigure}{.49\linewidth}
  \centering
  \includegraphics[width=\linewidth]{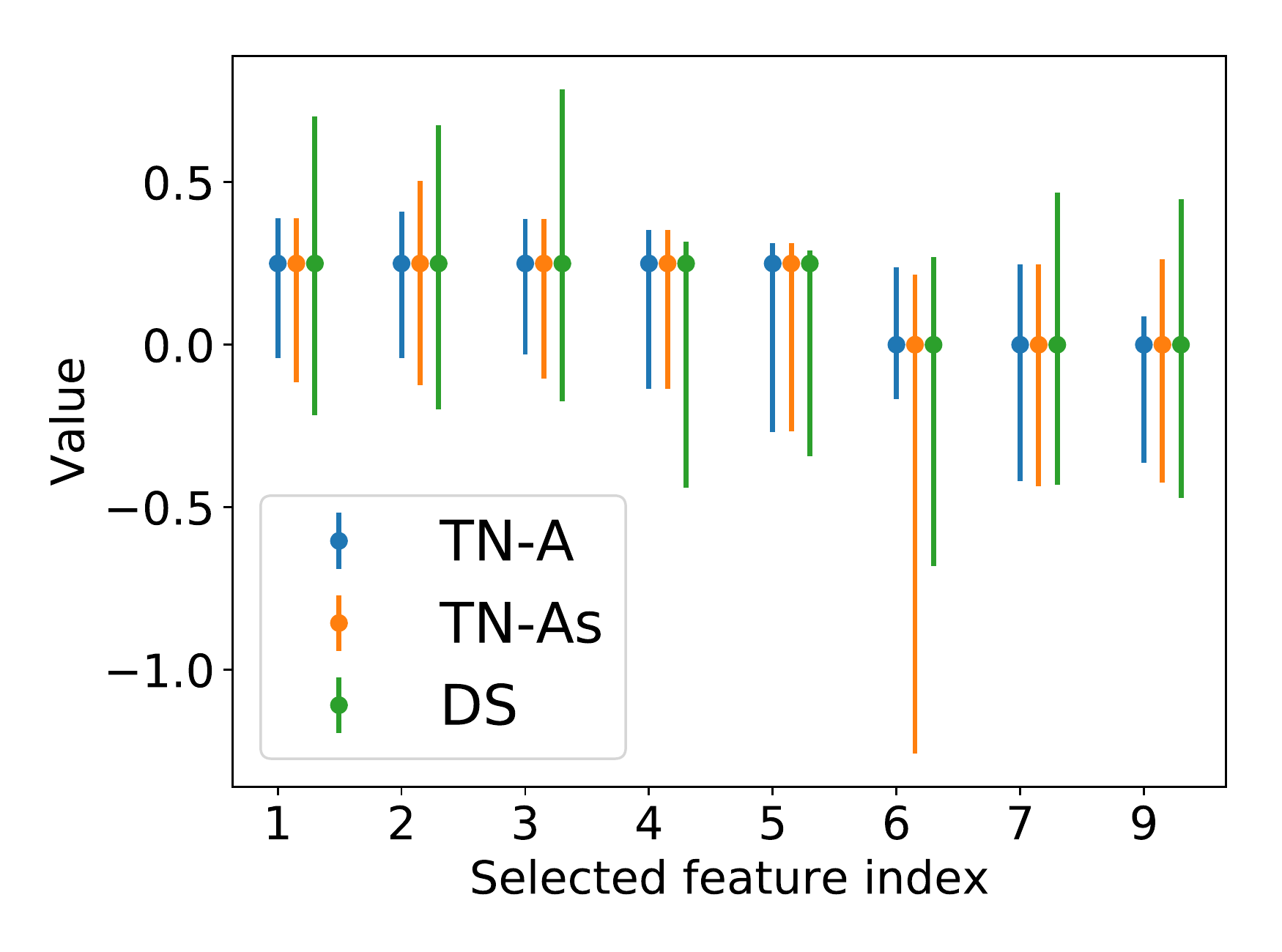}  
  \caption{CI demonstration}
\end{subfigure}
\begin{subfigure}{.49\linewidth}
  \centering
  \includegraphics[width=\linewidth]{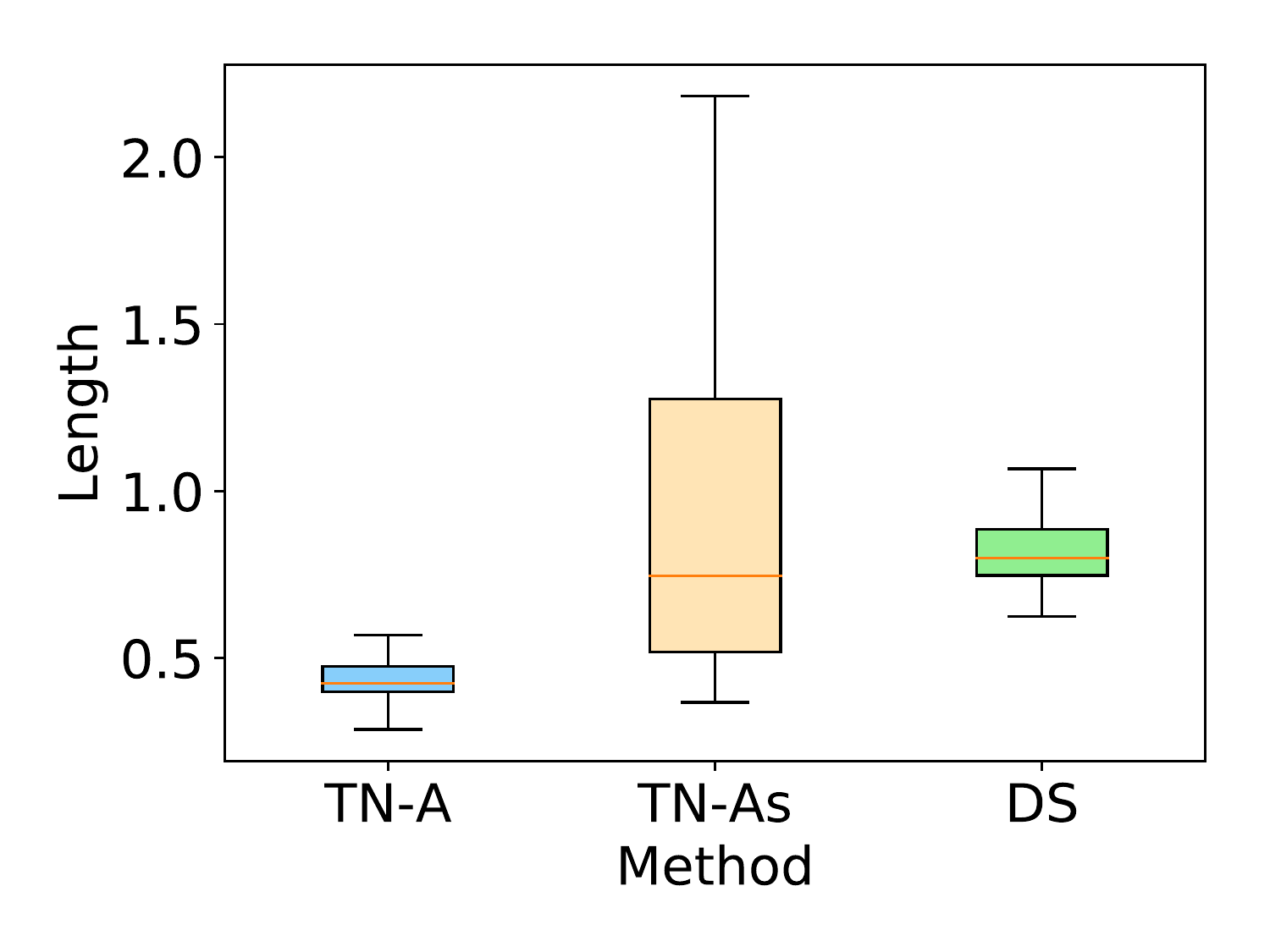}  
  \caption{Length of CI}
\end{subfigure}
\caption{Results of false positive rate (FPR) control, true positive rate (TPR) and confidence interval (CI).} 
\label{fig:fig_fpr_tpr_ci}
\end{figure}

\paragraph{The results of FPRs, TPRs and CIs.} The results of FPR and TPR are shown in Figures \ref{fig:fig_fpr_tpr_ci}a and \ref{fig:fig_fpr_tpr_ci}b.
In three cases, the FPRs are properly controlled under the significance level $\alpha$.
Regarding the TPR comparison, it is obvious that TN-A has the highest power.
In regard to CI experiments, we note that the number of selected features between Lasso and DS can be different.
Therefore, for a fair comparison, we only consider the features that are selected in both methods.
In our experiments, since 9 features were selected by the Lasso in the cases of TN-A and TN-As while only 8 features were selected in the case of DS, we only show the 95\% CI of the features that are selected in both cases in Figure \ref{fig:fig_fpr_tpr_ci}c.
The lengths of CI obtained by TN-A are almost the shortest.
We repeated this experiment 100 times and showed the boxplot of the lengths of the confidence intervals in Figure \ref{fig:fig_fpr_tpr_ci}d.
In summary, the CI results are consistent with the TPR results, i.e.,
TN-A has the shortest length of CI which indicates it has the highest power.




\begin{figure}[!t]
  \centering
  \includegraphics[width=.85\linewidth]{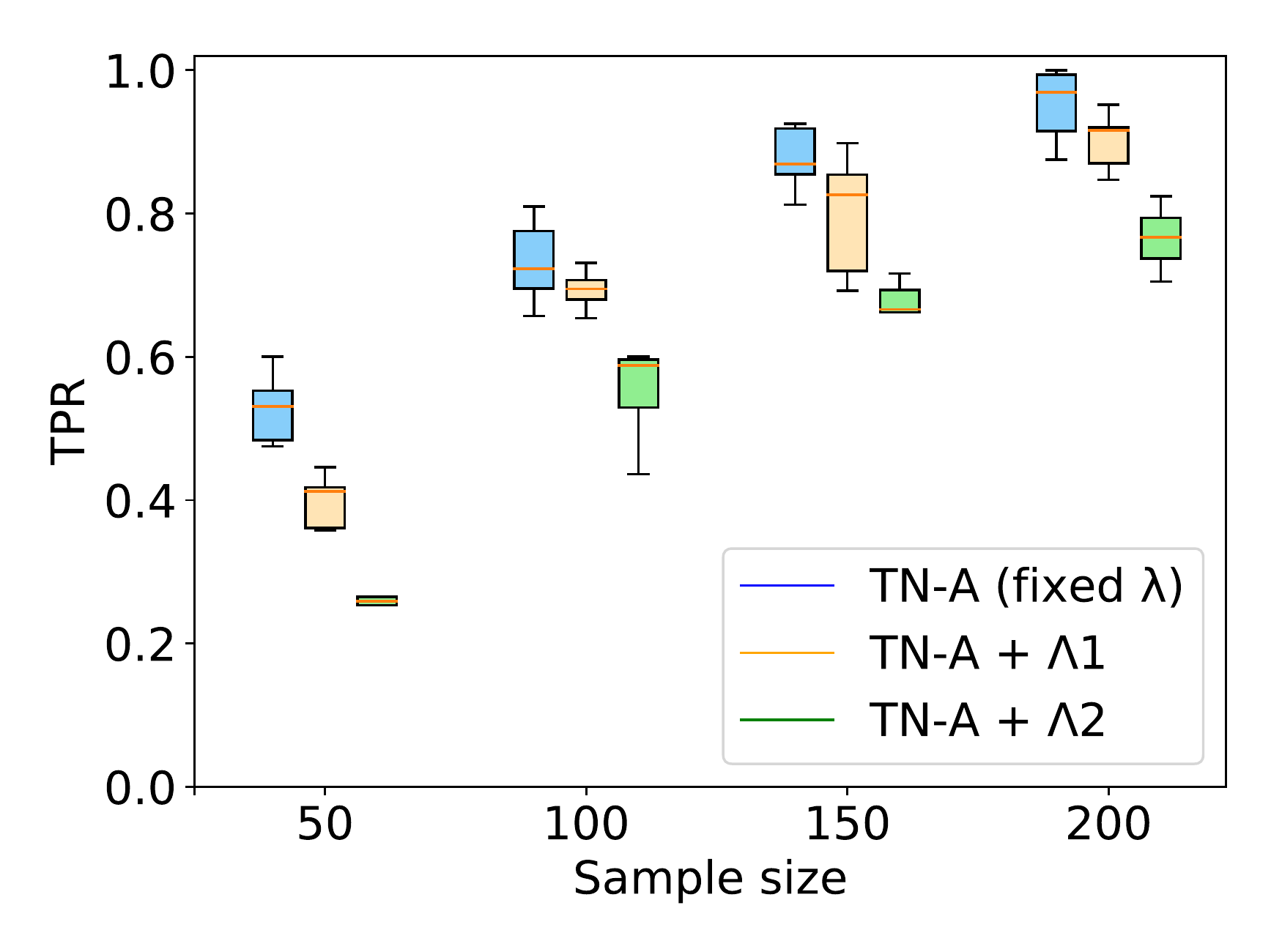}  
\caption{Demonstration of TPR when accounting cross-validation selection event.} 
\label{fig:tpr_cv}
\end{figure}

\paragraph{The results when accounting CV selection event.} We also demonstrate the TPRs and the lengths of CIs between the case when $\lambda = 2^0$ is fixed and $\lambda$ is selected from the set $\Lambda_1 = \{2^{-1}, 2^{0}, 2^{1}\}$ or $\Lambda_2 = \{2^{-10}, 2^{-9}, ..., 2^{9}, 2^{10}\}$.
%
%
We show that the TPR tends to decrease when increasing the size of $\Lambda$ as shown in Figure \ref{fig:tpr_cv}.
This is due to the fact that when we increase the size of $\Lambda$, we have to condition on more information which leads to shorter truncation interval and results low TPR.
The TPR results are consistent with the CI results shown in Figure \ref{fig:ci_cv} in which the length of CI is longer when increasing the size of $\Lambda$.
Besides, we also conducted TPR comparison between our method and the \emph{over-conditioning} version proposed in \cite{loftus2015selective}. 
The results are shown in Figure \ref{fig:fig_tpr_cv_para_oc}.
Our method has higher power since we can characterize minimum amount of conditioning.

\begin{figure}[t]
  \centering
  \includegraphics[width=.85\linewidth]{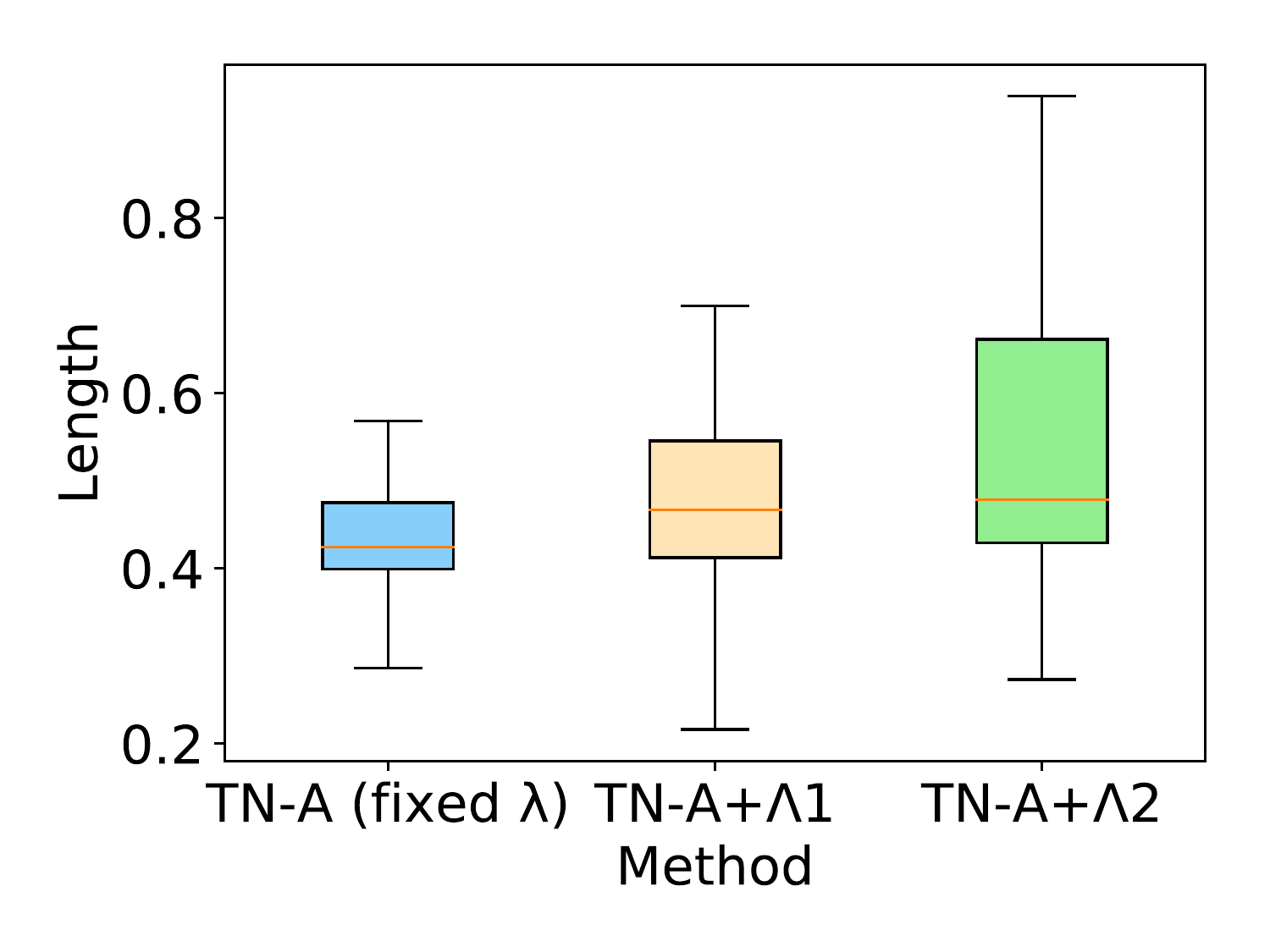}  
\caption{Demonstration of CI length when considering cross-validation selection event.} 
\label{fig:ci_cv}
\end{figure}

\begin{figure}[t]
\begin{subfigure}{.495\linewidth}
  \centering
  \includegraphics[width=\linewidth]{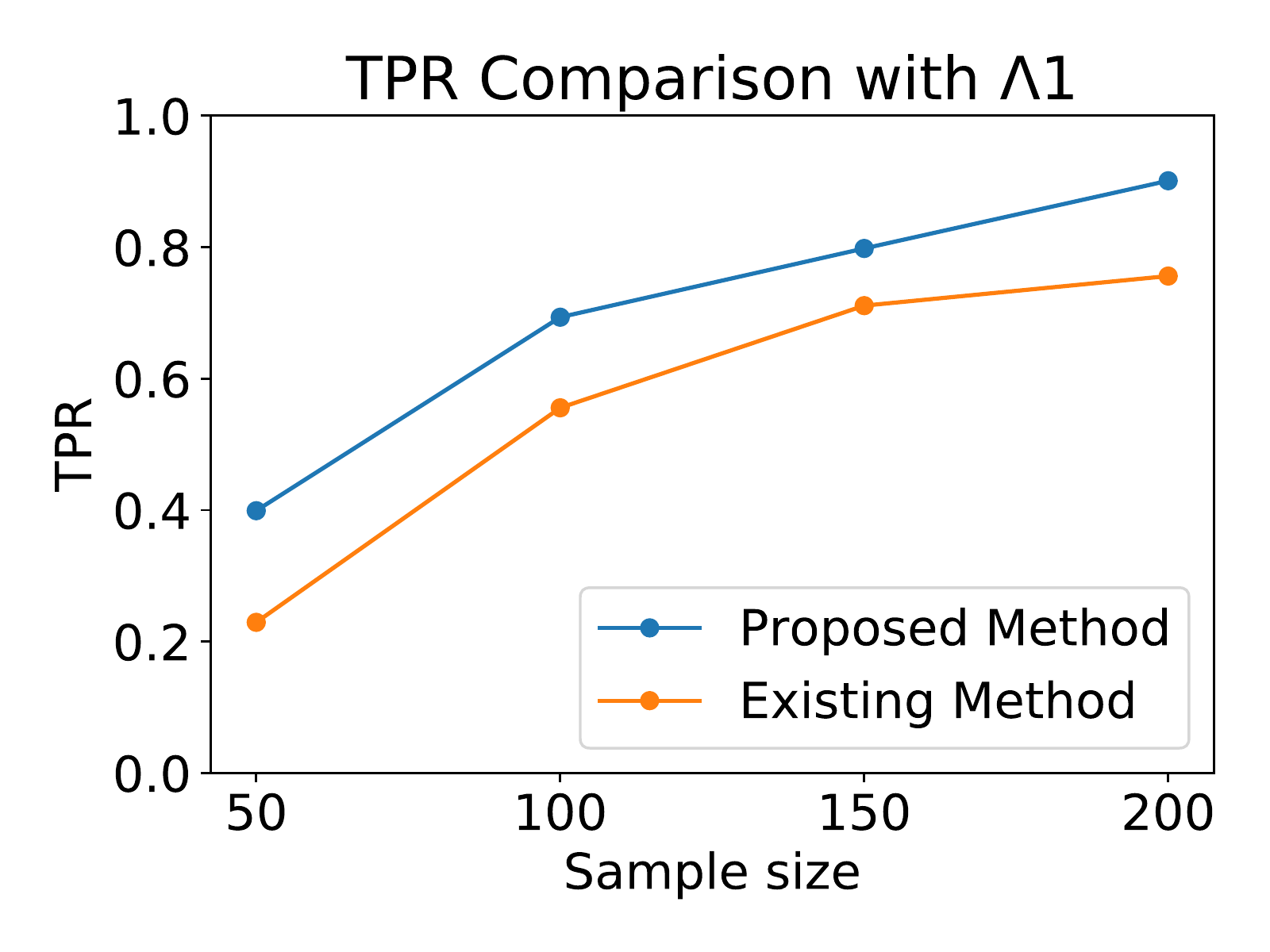}  
  \caption{$\Lambda_1 = \{2^{-1}, 2^0, 2^1\} $}
\end{subfigure}
\begin{subfigure}{.495\linewidth}
  \centering
  \includegraphics[width=\linewidth]{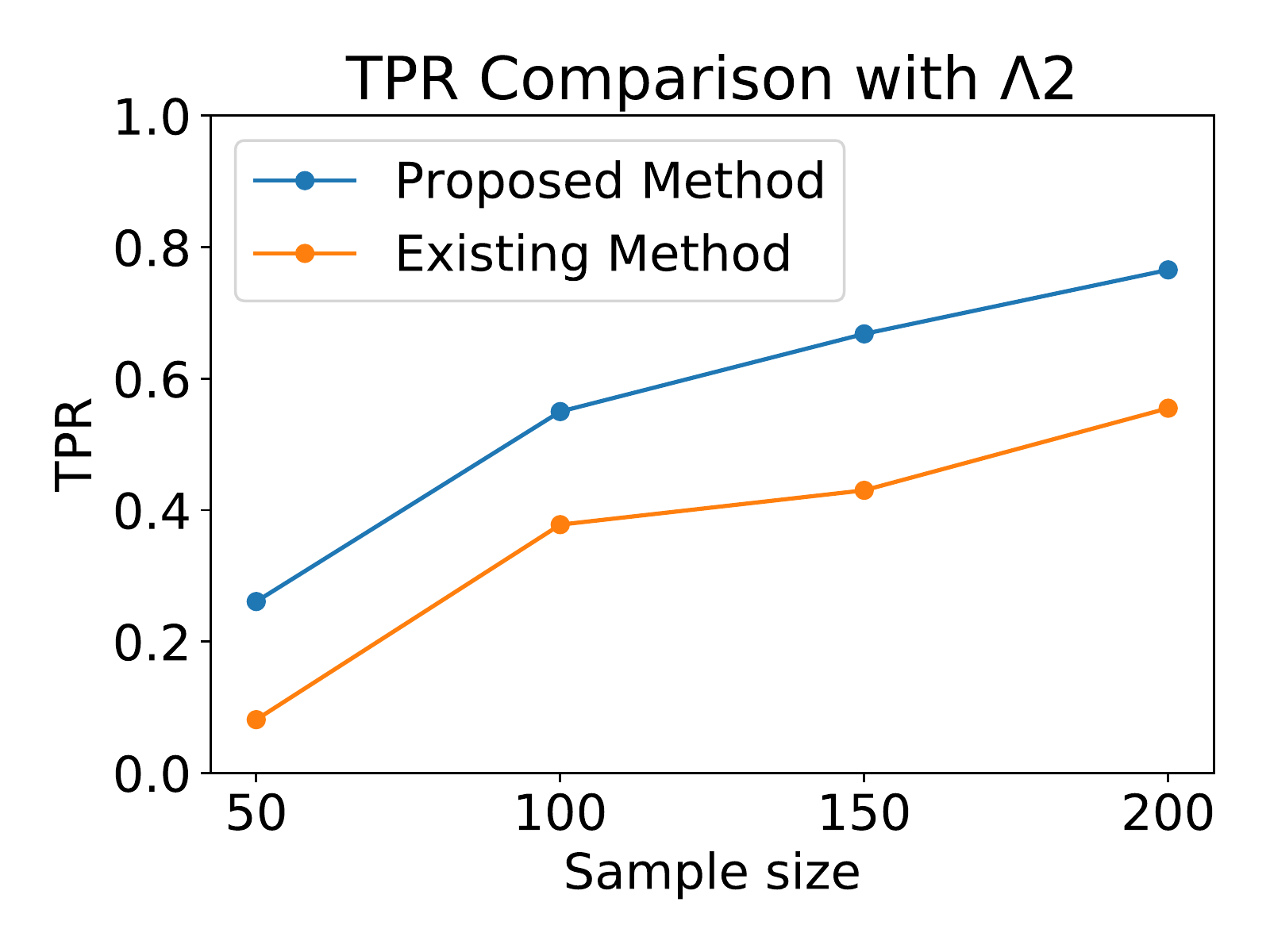}  
  \caption{$\Lambda_2 = \{2^{-10},..., 2^{10}\} $}
\end{subfigure}
\caption{TPR comparison with the existing method \citep{loftus2015selective} when accounting CV selection event.} 
\label{fig:fig_tpr_cv_para_oc}
\end{figure}


\paragraph{The efficiency of the proposed method.} In \cite{lee2016exact}, the authors mentioned the \emph{naive} way to remove sign conditioning by enumerating all possible combination of signs $2^{|\cA_{\rm obs}|}$ which is only feasible when $|\cA_{\rm obs}|$ is small.
On the left-hand side of Figure \ref{fig:cc}, we show the efficiency of our method compared to the naive way of removing sign conditioning.
On the right-hand side of Figure \ref{fig:cc}, the Lasso SI without conditioning on signs can be done even when $n=10,000$, $p=10,000$ and 
thousands of features are selected while the naive way can not finish the task in realistic time.
We also additionally show the efficiency of our method compare to two methods in  \cite{liu2018more}, which we call TN-$\ell_1$ and TN-Custom.
The details of these two methods are shown in Appendix \ref{ext:partial_target}.
In general, to perform these two methods, we still need to \emph{naively} enumerate all possible combinations of signs.
The results are shown in Figure \ref{fig:cc_vs_liu}.

\begin{figure}[t]
\begin{subfigure}{.47\linewidth}
  \centering
  \includegraphics[width=0.93\linewidth]{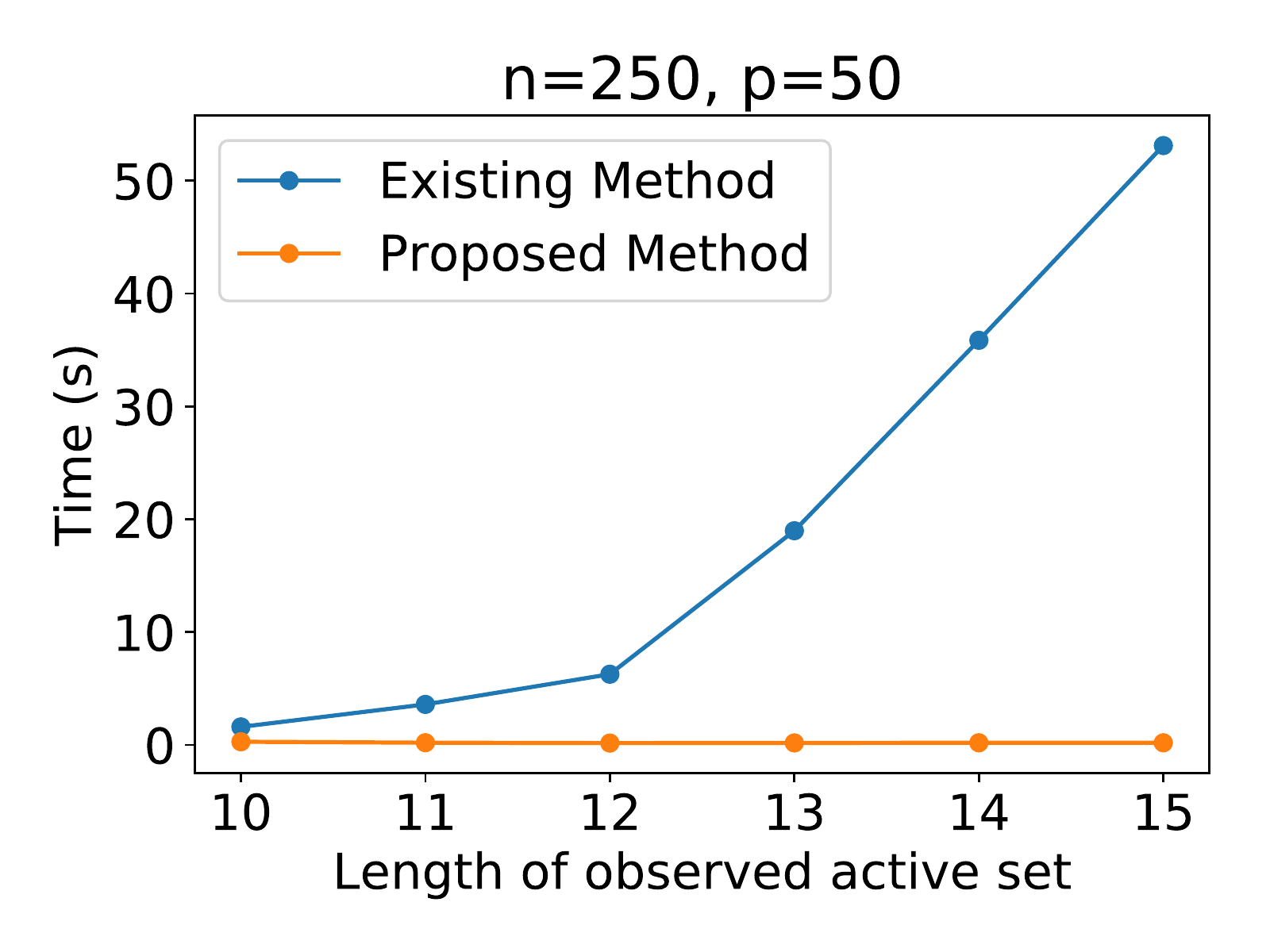}  
\end{subfigure}
\begin{subfigure}{.52\linewidth}
  \centering
  \includegraphics[width=\linewidth]{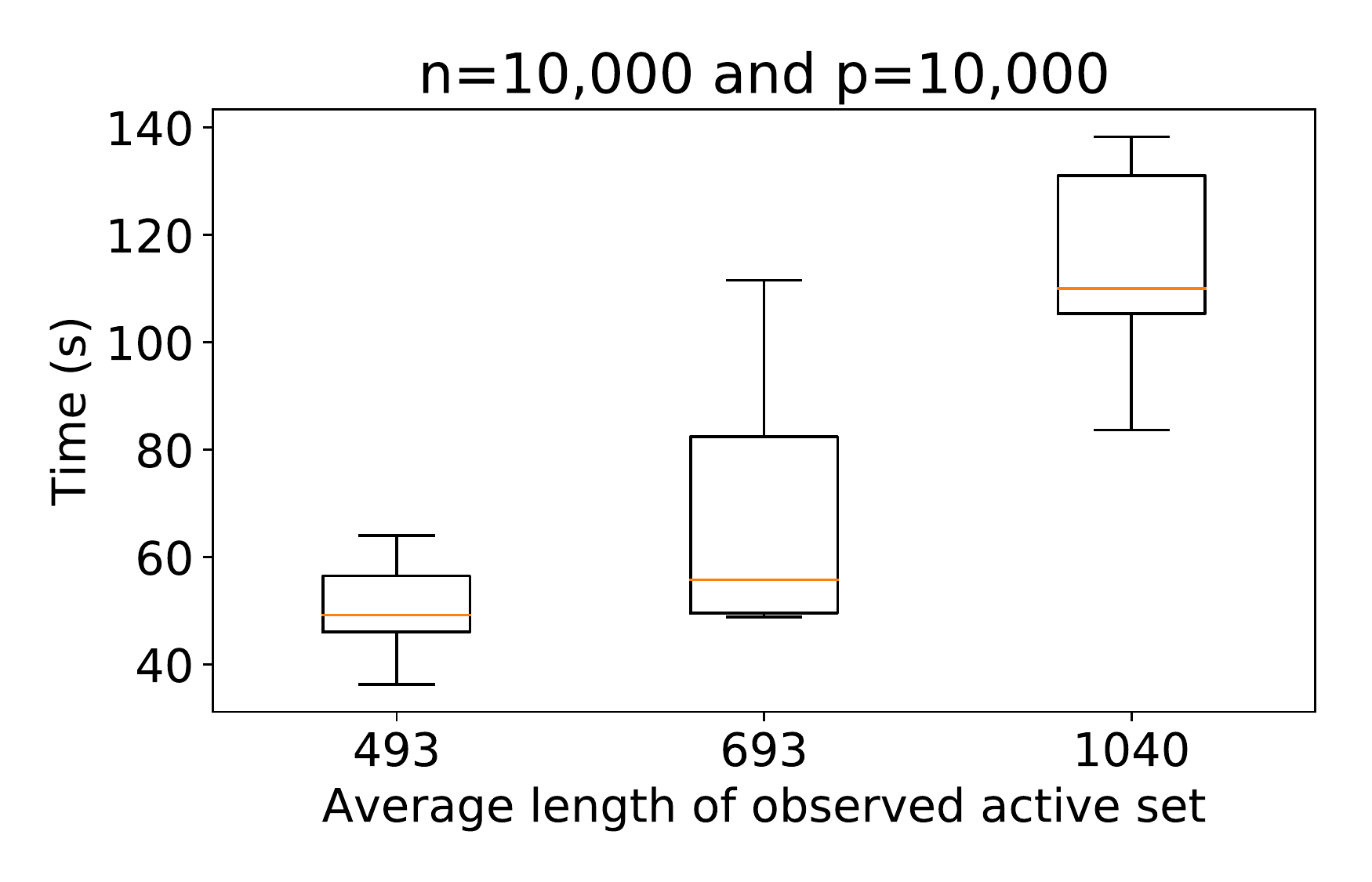} 
\end{subfigure}
\caption{Efficiency of the proposed method. With our method, Lasso SI without conditioning on signs can be done even when thousands of features are selected.} 
\label{fig:cc}
\end{figure}

\begin{figure}[t]
\begin{subfigure}{.49\linewidth}
  \centering
  \includegraphics[width=\linewidth]{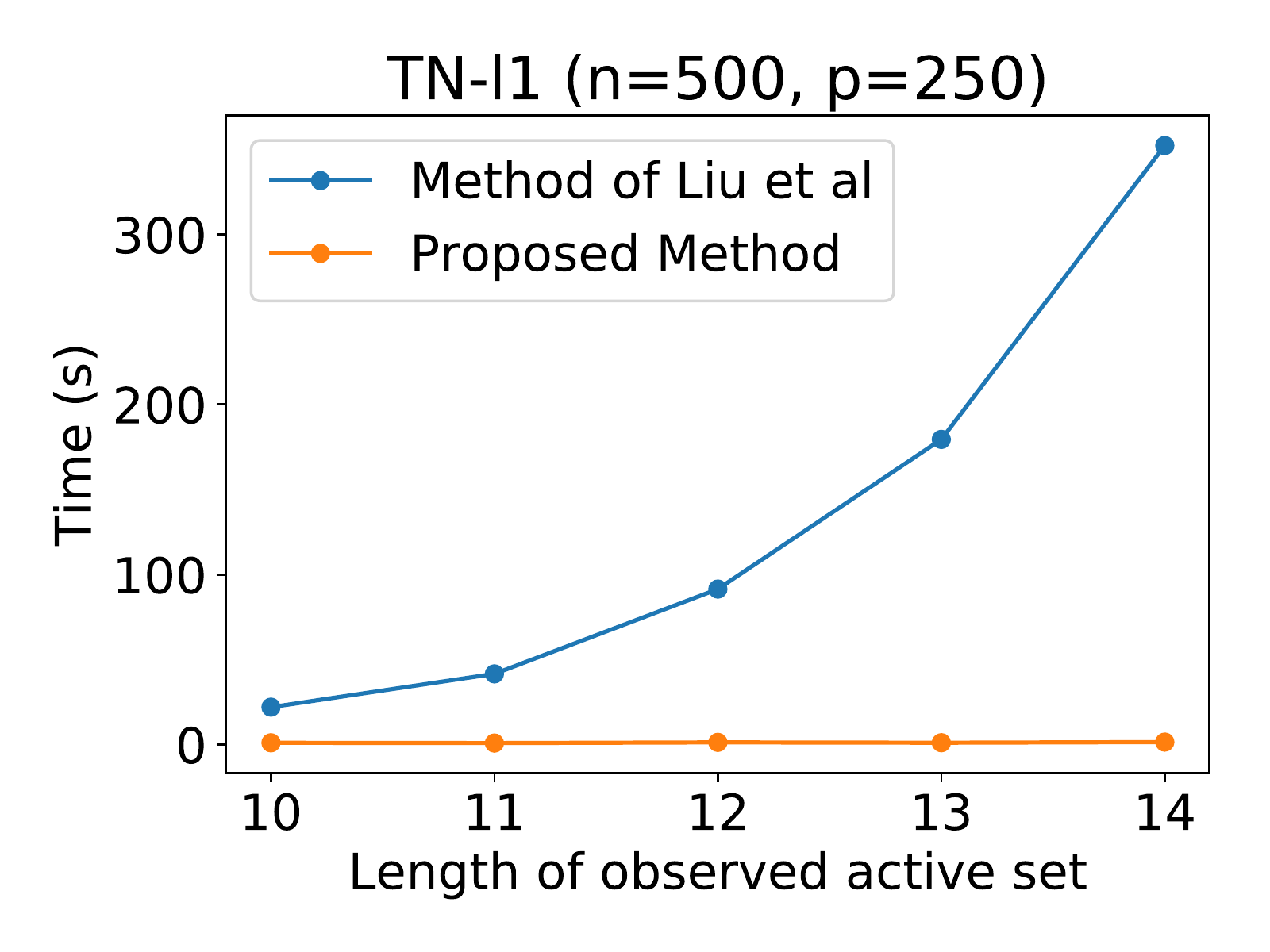}  
\end{subfigure}
\begin{subfigure}{.49\linewidth}
  \centering
  \includegraphics[width=\linewidth]{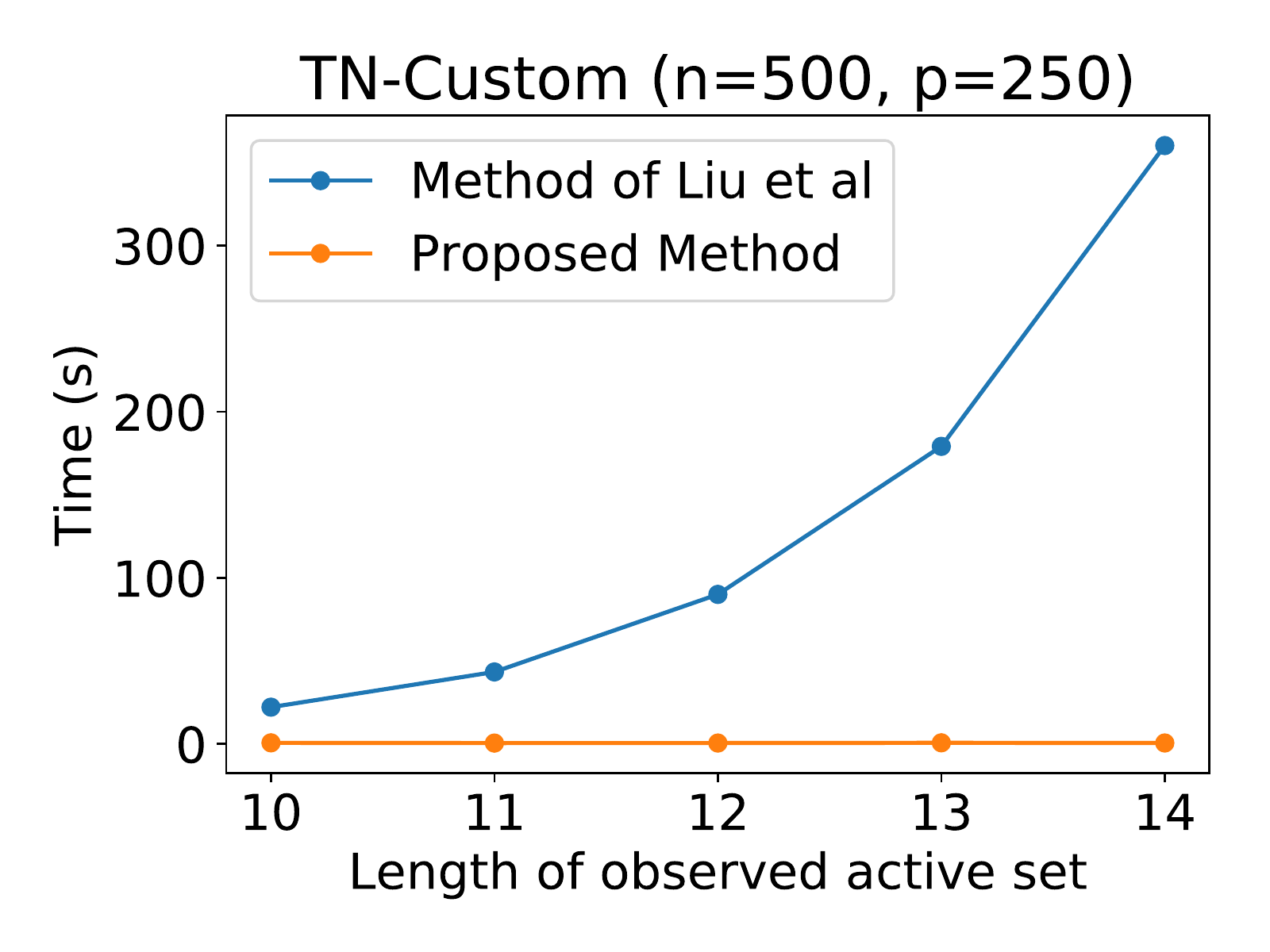} 
\end{subfigure}
\caption{Comparison between the proposed method and methods in  \cite{liu2018more}, in which an exponentially increasing number of all possible sign combinations are still required.} 
\label{fig:cc_vs_liu}
\end{figure}


\begin{figure}[!t]
  \centering
  \includegraphics[width=0.85\linewidth]{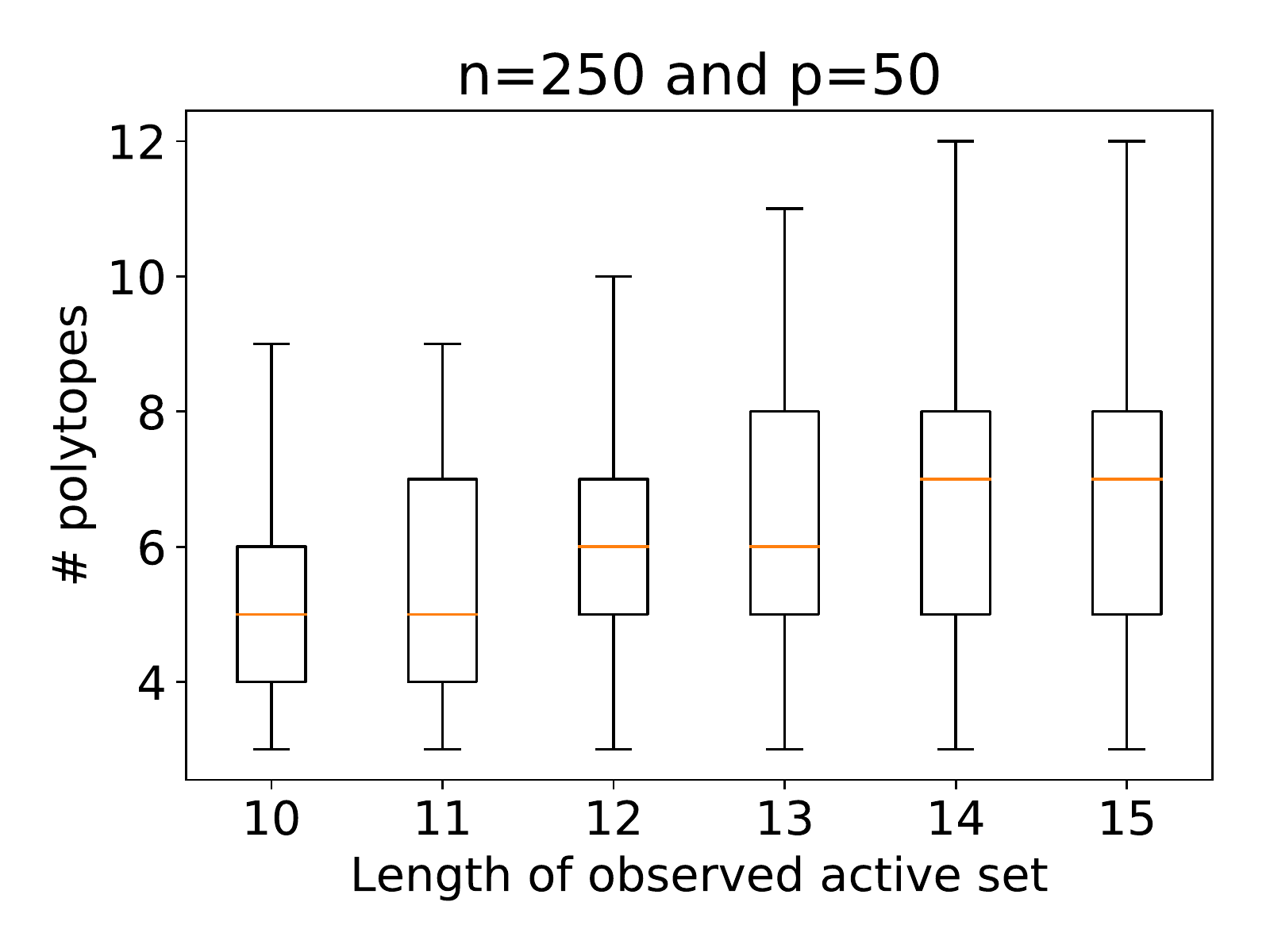}  
\caption{Number of encountered intervals on the line.} 
\label{fig:no_encounted}
\end{figure}

One might wonder how we can circumvent the computational bottleneck of exponentially increasing number of polytopes. 
Our experience suggests that, by focusing on the the line along the test-statistic in data space, we can skip majority of the polytopes that do not affect the truncated Normal sampling distribution because they do not intersect with this line.
In other words, we can skip majority of combinations of signs that never appear.

In Figure \ref{fig:no_encounted}, we show the boxplot of the actual number of intervals of $z$ that we encountered on the line when constructing the truncation region $\cZ$. 
This indicates that the number of polytopes intersecting the line $z$ that we need to consider is much smaller than $2^{|\cA_{\rm obs}|}$, which is considered in \cite{lee2016exact}---this is the reason why the proposed approach can resolve all major limitations of the current SI method, making Lasso SI more powerful and practical.

We did not compare the computational time between the proposed method TN-A and the over-conditioning version TN-As because TN-As is obviously faster than TN-A but it has lower power than TN-A.
Our main purpose is to demonstrate that the proposed method  not only has high statistical power but also has practically computational costs.

We note that, in the worst-case, the complexity of the proposed method still grows exponentially. 
This is a common issue in other parametric programming applications such as regularization paths. 
However, fortunately, it has been well-recognized that this worst case rarely happens in practice, and our experiments suggest that this also applies to PP-based SI.

\paragraph{The robustness of the proposed method in terms of the FPR control.} We demonstrate the robustness of our method in terms of the FPR control by considering the following cases:

$\bullet$ Non-normal noise:  we consider the noise following Laplace distribution, skew normal distribution (skewness coefficient 10), and $t_{20}$ distribution.

$\bullet$ Unknown $\sigma^2$: we also consider the case  when the variance is estimated from the data.

%
We generated $n$ outcomes as $y_i = \bm x_i^\top \bm \beta + \veps_i$, 
$i = 1, ..., n$, 
where 
$p = 5, \bm x_i \sim \NN(0, I_p)$, 
and $\veps_i$ follows Laplace distribution, skew normal distribution, or $t_{20}$ distribution with zero mean and standard deviation was set to 1.
In the case of estimated $\sigma^2$, $\veps_i \sim \NN(0, 1)$.
We set all elements of $\bm \beta$ to 0, and set $\lambda = 0.5$.
For each case, we ran 1,200 trials for each $n \in \{100, 200, 300, 400\}$.
We confirmed that our method still maintains good performance on FPR control. The results are shown in Appendix \ref{appendix:exp_details}.

\subsection{Results on Real-World Datasets}

We demonstrate the efficiency of the proposed method by applying it on high-dimensional real-world bioinformatics related datasets, which is available at \url{http://www.coepra.org/CoEPrA_regr.html}.
In datasets 1 and 3, $n$ is the number of nona-peptides. 
Each amino acid in a nona-peptide is described by 643 descriptors, for a total of $p = 643 \times 9 = 5787$ descriptors.
In dataset 2, $n$ is the number of octa-peptides. 
Each amino acid in a octa-peptide is described by 643 descriptors, for a total of $p = 643 \times 8 = 5144$ descriptors.
For these experiments, we used elastic net instead of Lasso to obtain large $\cA_{\rm obs}$.
The extension of the proposed method for elastic net is presented in Appendix \ref{ext:elastic_net}.
The results are shown in Table \ref{tab:real_world}.
The time shown in the table is the average time to compute $p$-value for a selected feature.

\begin{table}[!t]
\caption{Results on high-dimensional real-world bioinformatics related datasets.} \label{tab:real_world}
\begin{center}
\renewcommand{\arraystretch}{1.5}
\begin{tabular}{ |c|c|c|c|c| } 
 \hline
  & $n$ & $p$ & $|\cA_{\rm obs}|$ & Avg. Time (s) \\ 
 \hline
 \hline
 Dataset 1& 89 & 5787 & \textbf{600} & 0.374 \\ 
 \hline
 Dataset 2& 76 & 5144 & \textbf{621} & 0.344 \\ 
 \hline
 Dataset 3& 133 & 5787 & \textbf{660} & 0.342 \\ 
 \hline
\end{tabular}
\end{center}
\end{table}

\section{Conclusion}
In this paper, we have proposed a general method for characterizing the selection event of Lasso SI by introducing piecewise-linear parametric programing approach.
With the proposed method, we can conduct a powerful SI by conditioning only on the selected features without the need of enumerating all possible sign vectors.
Besides, we also introduced a new way to charactering the cross-validation based tuning parameter selection.
The proposed method not only overcomes the drawbacks of current Lasso SI methods but also improves the performance and practicality of SI for Lasso in various respects. 
Our idea is general and can be applied to circumvent several drawbacks of all the methods that are based on the current SI framework.
%
%
We conducted experiments on both synthetic and real-world datasets to demonstrate the effectiveness and efficiency of our proposed method.


\subsubsection*{Acknowledgements}
This work was partially supported by MEXT KAKENHI (20H00601, 16H06538), JST CREST (JPMJCR1502), RIKEN Center for Advanced Intelligence Project, and RIKEN Junior Research Associate Program.

\bibliographystyle{abbrvnat}
\bibliography{ref}

\newpage
\onecolumn

\section{Appendix}

\subsection{Detailed Proof for Lemma \ref{lemma:transition_point}} \label{appendix:proof_lemma_2}

From Equation (\ref{eq:lemma1_eq1}), we can see that $\hat{\bm \beta}_{\cA_z}(z)$ is a function of $z$.
For a real value $z$, there exists $t^1_{z}$ such that for any real value $z^\prime$ in $[z, z + t^1_{z})$, all elements of $\hat{\bm \beta}_{\cA_{z^\prime}}(z^\prime)$ remain the same signs with $\hat{\bm \beta}_{\cA_z}(z)$.
Similarly, from Equation (\ref{eq:lemma1_eq2}), we can see that $\bm s_{\cA^c_z}(z)$ is a function of $z$.
Then, for a real value $z$, there exists $t^2_{z}$ such that for any real value $z^\prime$ in $[z, z + t^2_{z})$, all elements of $\bm s_{\cA^c_{z^\prime}}(z^\prime)$ are smaller than 1 in absolute value.
Finally, by taking $t_z = \min \{t^1_{z}, t^2_{z}\}$, we obtain the interval in which the active set and signs of lasso solution remain the same.
The remaining task is to compute $t^1_{z}$ and $t^2_{z}$.

We first show how to derive $t^1_{z}$. From Equation (\ref{eq:lemma1_eq1}), we have
\begin{align*}
	\hat{\bm \beta}_{\cA_z}(z^\prime) - \hat{\bm \beta}_{\cA_z}(z) 
	= 
	{\bm \psi}_{\cA_z}(z) \times (z^\prime - z).
\end{align*}
To guarantee $\hat{\bm \beta}_{\cA_z}(z^\prime)$ and $\hat{\bm \beta}_{\cA_z}(z) $ have the same signs, 
\begin{align} \label{appendix_eq:lemma_2_proof_c1}
	s_j (z^\prime) = s_j (z), \quad \forall j \in \cA_z.
\end{align}
For a specific $j \in \cA_z$, we consider the following cases:
\begin{itemize}
	\item If $\hat{\beta}_j(z) > 0$, then $\hat{\beta}_j(z^\prime) =\hat{\beta}_j(z) + {\psi}_j(z) \times (z^\prime - z) > 0$.
	\begin{itemize}
		\item If ${\psi}_j(z) > 0$, then $ z^\prime - z > - \frac{ \hat{\beta}_j(z)}{{\psi}_j(z) }$ (This inequality always holds since the left hand side is positive while the right hand side is negative).
		\item If ${\psi}_j(z) < 0$, then $ z^\prime - z < - \frac{ \hat{\beta}_j(z)}{{\psi}_j(z)}$.
	\end{itemize}

	\item If $\hat{\beta}_j(z) < 0$, then $\hat{\beta}_j(z^\prime) =\hat{\beta}_j(z) + {\psi}_{j}(z) \times (z^\prime - z) < 0$.
	\begin{itemize}
		\item If ${\psi}_j(z) > 0$, then $ z^\prime - z < - \frac{ \hat{\beta}_j(z)}{{\psi}_j(z)}$.
		\item If ${\psi}_j(z) < 0$, then $ z^\prime - z > - \frac{ \hat{\beta}_j(z)}{{\psi}_j(z)}$ (This inequality always holds since the left hand side is positive while the right hand side is negative).
	\end{itemize}
\end{itemize}
Finally, for satisfying the condition in Equation (\ref{appendix_eq:lemma_2_proof_c1}),
\begin{align*}
	z^\prime - z < \min \limits_{j \in \cA_z} \left( - \frac{\hat{\beta}_j(z)}{\psi_j(z)} \right)_{++}  = t^1_{z}.
\end{align*}

We next show how to derive $t^2_{z}$. From Equation (\ref{eq:lemma1_eq2}), we have
\begin{align*}
	\lambda {\bm s}_{\cA^c_z}(z^\prime) - \lambda {\bm s}_{\cA^c_z}(z) 	
	=
	{\bm \gamma}_{\cA^c_z}(z) \times (z^\prime - z).
\end{align*}
To guarantee $\|\lambda {\bm s}_{\cA^c_z}(z^\prime)\|_{\infty} = \|\lambda {\bm s}_{\cA^c_z}(z) + {\bm \gamma}_{\cA^c_z}(z) \times (z^\prime - z)\|_{\infty} < \lambda$,
\begin{align} \label{appendix_eq:lemma_2_proof_c2}
	-\lambda < \lambda s_j(z) + \gamma_j(z) \times (z^\prime - z) < \lambda, \quad \forall j \in \cA^c_z.
\end{align}
For a specific $j \in \cA^c_z$, we have the following cases:
\begin{itemize}
	\item If $\gamma_j(z) > 0$, then 
	$\frac{-\lambda - \lambda s_j(z)} {\gamma_j(z)} < z^\prime - z < \frac{\lambda - \lambda s_j(z)} {\gamma_j(z)}$.
	\item If $\gamma_j(z) < 0$, then 
	$\frac{\lambda - \lambda s_j(z)} {\gamma_j(z)} < z^\prime - z < \frac{-\lambda - \lambda s_j(z)} {\gamma_j(z)}$.
\end{itemize}
Note that the first inequalities of the above two cases always hold since the left hand side 
is negative while the right hand side is positive).
Then, for satisfying the condition in Equation (\ref{appendix_eq:lemma_2_proof_c2}),
\begin{align*}
	z^\prime - z < \min \limits_{j \in \cA^c_z} \left( \lambda \frac{{\rm sign}(\gamma_j(z)) - s_j(z)}{\gamma_j(z)} \right)_{++} = t^2_{z}.
\end{align*}
Finally, we can compute $t_z$ by taking $t_z = \min \left \{ t^1_{z}, t^2_{z} \right \}$.


\subsection{Derivations of the Proposed Method for Various Settings}
\subsubsection{Elastic Net} \label{ext:elastic_net}
In some cases, the lasso solutions are unstable. 
One way to stabilize them is to add an $\ell_2$ penalty to the objective function, resulting in the elastic net \citep{zou2005regularization}.
Therefore, we extend our proposed method and provide detailed derivation for testing the selected features in elastic net case.
We now consider the optimization problem with parametrized response vector ${\bm y}(z)$ for $z \in \RR$ as follows
\begin{equation} \label{eq:parametric_elastic_net}
	\hat{{\bm \beta}}(z) = \argmin \limits_{{\bm \beta} \in \RR^p} \frac{1}{2n} \|{\bm y}(z) - X {\bm \beta}\|^2_2 + \lambda \|{\bm \beta}\|_1 + \frac{1}{2} \delta \|{\bm \beta}\|^2_2.
\end{equation}
For any $z$ in $\RR$, the optimality condition is given by
\begin{align} \label{eq:opt_condition_en}
	\frac{1}{n} X^\top\left ( X \hat{{\bm \beta}}(z) - {\bm y}(z) \right ) + \lambda {\bm s}(z) + \delta \hat{{\bm \beta}}(z)= 0, \ {\bm s}(z) \in \partial \|\hat{{\bm \beta}}(z)\|_1.
\end{align}

Similar to lasso case, to construct the truncation region $\cZ$, we have to 1) compute the entire path of $\hat{{\bm \beta}}(z)$ in Equation (\ref{eq:parametric_elastic_net}), and 2) identify a set of intervals of $z$ on which $\cA({\bm y}(z)) = \cA({\bm y}^{\rm obs})$.

\begin{lemma}
\label{lemma:piecewise_linear_elastic_net}
Let us consider two real values $z^\prime$ and $z$ $(z^\prime > z)$. 
If $\hat{\bm \beta}_{\cA_z}(z)$  and $\hat{\bm \beta}_{\cA_{z^\prime}}(z^\prime)$ have the same active set and the same signs, then we have
\begin{align} 
	\hat{\bm \beta}_{\cA_z}(z^\prime) - \hat{\bm \beta}_{\cA_z}(z) 
	&= 
	{\bm \psi}_{\cA_z}(z) \times (z^\prime - z), \label{eq:lemma1_eq1_en}\\
	\lambda {\bm s}_{\cA^c_z}(z^\prime) - \lambda {\bm s}_{\cA^c_z}(z) 	&=
	{\bm \gamma}_{\cA^c_z}(z) \times (z^\prime - z), \label{eq:lemma1_eq2_en}
\end{align}
where 
${\bm \psi}_{\cA_z}(z) = (X^\top_{\cA_z} X_{\cA_z} + n \delta I_{|\cA_z|})^{-1} X^\top_{\cA_z} {\bm b}$, 
and 
${\bm \gamma}_{\cA^c_z}(z) = \frac{1}{n} ( X^\top_{\cA^c_z} {\bm b} - X^\top_{\cA^c_z} X_{\cA_z} {\bm \psi}_{\cA_z}(z) )$.
\end{lemma}

\begin{proof}
From the optimality conditions of the elastic net (\ref{eq:opt_condition_en}) , we have 
\begin{align}
	(X^\top_{\cA_z} X_{\cA_z} + n \delta I_{|\cA_z|})\ \hat{\bm \beta}_{\cA_z}(z) 
	- X^\top_{\cA_z} {\bm y}(z) 
	+ n \lambda {\bm s}_{\cA_z}(z) = 0, 			
	\label{eq:proof_eq1_en} \\
	(X^\top_{\cA_{z^\prime}} X_{\cA_{z^\prime}} + n \delta I_{|\cA_{z^\prime}|})\ \hat{\bm \beta}_{\cA_{z^\prime}}(z^\prime) 
	- X^\top_{\cA_{z^\prime}} {\bm y}(z^\prime) 
	+ n \lambda {\bm s}_{\cA_{z^\prime}}(z^\prime) = 0.
	\label{eq:proof_eq2_en} 
\end{align}
By substracting (\ref{eq:proof_eq1_en}) from (\ref{eq:proof_eq2_en}) and $\cA_z = \cA_{z^\prime}$, we have
\begin{align*}
	\hat{\bm \beta}_{\cA_z}(z^\prime) - \hat{\bm \beta}_{\cA_z}(z) 
	&=
	(X^\top_{\cA_z} X_{\cA_z} + n \delta I_{|\cA_z|})^{-1} X^\top_{\cA_z} ({\bm y}(z^\prime) - {\bm y}(z))\\
	&=(X^\top_{\cA_z} X_{\cA_z} + n \delta I_{|\cA_z|})^{-1} X^\top_{\cA_z} ({\bm a} + {\bm b} z^\prime - {\bm a} - {\bm b} z)\\
	&=(X^\top_{\cA_z} X_{\cA_z} + n \delta I_{|\cA_z|})^{-1} X^\top_{\cA_z} {\bm b} \times (z^\prime - z).
\end{align*}
Thus, we achieve Equation (\ref{eq:lemma1_eq1_en}).
Similarly, we can write the optimality conditions with $X_{\cA^c_z}$ for $z$ and $z^\prime$, and easily obtain Equation (\ref{eq:lemma1_eq2_en}).
\end{proof}
Now, we can see that $\hat{\bm \beta}_{\cA_z}(z)$ and ${\bm s}_{\cA^c_z}(z)$ are functions of $z$. Then, for a real value $z$, there exists $t_z$ such that for any real value $z^\prime$ in $[z, z + t_z)$, 
all elements of $\hat{\bm \beta}_{\cA_{z^\prime}}(z^\prime)$ remain the same signs with $\hat{\bm \beta}_{\cA_z}(z)$,
and
all elements of $\bm s_{\cA^c_{z^\prime}}(z^\prime)$ are strictly smaller than $1$ in absolute value.
The value of $t_z$ can be computed by Lemma \ref{lemma:transition_point} as in lasso case.

\subsubsection{Full Target Case} \label{ext:full_target}
In the full target case, as discussed in \cite{liu2018more}, the data is used to choose the interesting features but it is \emph{not} used for summarizing the relation between the response and the selected features.
Therefore, we can always use \emph{all} the features to define the direction of interest 
\[
	\bm \eta_j = X(X^\top X)^{-1} \bm e_j,
\]
where $\bm e_j \in \RR^p$ is a zero vector with one at its $j^{\rm th}$ coordinate.
The conditional inference is defined as
\begin{align} \label{eq:full_model_conditional_inference}
	\bm \eta_j^\top \bm {\bm Y} \mid \left \{ j \in \cA(\bm Y) , \bm q(\bm Y) = \bm q({\bm y}^{\rm obs}) \right \}.
\end{align}
In \cite{liu2018more}, the authors proposed a solution to conduct conditional inference for a specific case when $p < n$, and there is no solution for the case when $p > n$.
With the proposed parametric programming method, we can solve this problem.
We first re-write the conditional inference in (\ref{eq:full_model_conditional_inference}) as the problem of characterizing the sampling distribution of 
\begin{align} \label{eq:full_model_conditional_inference_parametric}
	Z \mid \{Z \in \cZ\} \text{ where } \cZ = \{z \in \RR \mid j \in \cA(\bm y(z))\}.
\end{align}
The $\bm y(z)$ in (\ref{eq:full_model_conditional_inference_parametric}) is defined as in (\ref{eq:parametrized_data_space}).
Then, to identify $\cZ$, we only need to obtain the path of Lasso solution $\hat{\bm{\beta}}(z)$ as we proposed in \S3, and simply check the intervals in which $j$ is an element of the active set corresponding to $\hat{\bm{\beta}}(z)$ along the path.
Finally, after having $\cZ$, we can easily compute the selective $p$-value or selective confidence interval.

\subsubsection{Stable Partial Target Case} \label{ext:partial_target}
In the stable partial target case, as discussed in \cite{liu2018more}, we only allow stable features to influence the formation of the test-statistic.
The stable features are those with very strong signals and we would not to miss out.
We will choose a set $\cH_{\rm obs}$ of stable features.
Then, for any $j \in \cH_{\rm obs}, j \in \cA_{\rm obs}$,
\[
	\bm \eta_j = X_{\cH_{\rm obs}} (X_{\cH_{\rm obs}}^\top  X_{\cH_{\rm obs}})^{-1} \bm e_j.
\]
And, for any $j \not \in \cH_{\rm obs}, j \in \cA_{\rm obs}$,
\[
	\bm \eta_j = X_{\cH_{\rm obs} \cup \{j\}} (X_{\cH_{\rm obs} \cup \{j\}}^\top  X_{\cH_{\rm obs} \cup \{j\}})^{-1} \bm e_j.
\]
We next show how to construct $\cH_{\rm obs}$ according to \cite{liu2018more}.


\paragraph{Stable target formation by setting higher value of $\lambda$ (TN-$\ell_1$).}
In this case, $\cH_{\rm obs}$ is the lasso active set but with a higher value of $\lambda$ than the one was used to select $\cA_{\rm obs}$.
We denote $\cH_{\rm obs} =  \cH({\bm y}^{\rm obs})$, the conditional inference is then defined as
\begin{align} \label{eq:stable_partial_model_l1_conditional_inference}
	\bm \eta_j^\top \bm {\bm Y} \mid \left \{ j \in \cA(\bm Y) ,  \cH (\bm Y) = \cH({\bm y}^{\rm obs}), \bm q(\bm Y) = \bm q({\bm y}^{\rm obs}) \right \}.
\end{align}
The main drawback of the method in \cite{liu2018more} is that they have to consider all $2^{|\cH_{\rm obs}|}$ sign vectors, which requires huge computation time when $|\cH_{\rm obs}|$ is large.
With our piecewise-linear homotopy computation, we can easily overcome this drawback.
We first re-write the conditional inference in (\ref{eq:stable_partial_model_l1_conditional_inference}) as the problem of characterizing the sampling distribution of 
\begin{align} 
	Z \mid \{Z \in \cZ\} \text{ where } \cZ = \{z \in \RR \mid j \in \cA(\bm y(z)), \cH(\bm y(z)) = \cH({\bm y}^{\rm obs})\}.
\end{align}
We now can easily identify $\cZ = \cZ_1 \cap \cZ_2$, where 
$\cZ_1 = \{z \in \RR \mid j \in \cA(\bm y(z))\}$ which is the same with full target case, 
and 
$\cZ_2 = \{z \in \RR \mid \cH(\bm y(z)) = \cH({\bm y}^{\rm obs})\}$ 
which we can simply obtain by using the proposed method in \S3 of the main paper.


\paragraph{Stable target formation by setting a cutoff value $c$ (TN-Custom).}
In this case, we choose $\cH_{\rm obs}$ by setting a cutoff value $c$ for choosing $\beta_j$ such that $|\beta_j| \geq c$  
\footnote{We note that our formulation is slightly different but more general than the one in \cite{liu2018more}.}. 
The set $\cH_{\rm obs}$ is defined as 
\begin{align*}
	\cH_{\rm obs} = \left \{ j \in \cA_{\rm obs}, |\beta_j|  \geq c \right\},
\end{align*}
where  $ \beta_j =  \bm e^\top_j (X_{\cA_{\rm obs}}^\top  X_{\cA_{\rm obs}})^{-1} X_{\cA_{\rm obs}}^\top  \bm y^{\rm obs} $.
We denote $\cH_{\rm obs} = \cH(\cA_{\rm obs}) \subset \cA_{\rm obs}$, the conditional inference is then formulated as 
\begin{align}\label{eq:stable_partial_model_custom_conditional_inference}
	\bm \eta_j^\top \bm Y \mid \left \{\cH (\cA(\bm Y)) = \cH (\cA_{\rm obs}),  \cA(\bm Y) = \cA_{\rm obs} \right \}.
\end{align}
The main drawback of the method in \cite{liu2018more} is that they still require conditioning on $\{\cA(\bm Y) = \cA_{\rm obs}\}$, which is computationally intractable when $|\cA_{\rm obs}|$ is large because the enumeration of $2^{|\cA_{\rm obs}|}$ sign vectors is required.
With our proposed method, we can easily overcome this drawback.

\subsubsection{Marginal Model} \label{ext:marginal_model}
In the case of marginal model, we can always decide a priori to investigate the marginal relationship between the column $j$ of feature matrix $X$ and the observed response vector $\bm y^{\rm obs}$ if $j$ is selected.
The conditional inference is defined as
\begin{align} \label{eq:marginal_model_conditional_inference}
	\bm \eta_j^\top \bm {\bm Y} \mid \left \{ j \in \cA(\bm Y) , \bm q(\bm Y) = \bm q({\bm y}^{\rm obs}) \right \},
\end{align}
where $\bm \eta_j = X_j (X_j^\top X_j)^{-1} \bm e_j$.
The solution for conducting this conditional inference is the same with the full target case.
The only difference between marginal model case and full target case is the formulation of $\bm \eta_j$.

\subsubsection{Interaction Model} \label{ext:interaction_model}
Firstly, we apply Lasso on $\{X, \bm y^{\rm obs}\}$ to obtain the active set $\cA_{\rm obs} = \cA(\bm y^{\rm obs})$.
Next, we construct a feature matrix for interaction model as 
\[X_{\rm inter} = (X_i X_j)_{i, j \in \cA_{\rm obs}, i < j} \in \RR^{n \times d},
\]
where $d = 0.5 |\cA_{\rm obs}| (|\cA_{\rm obs}| - 1)$.
Then, the Lasso optimization problem for the interaction model is given by
\begin{align*}
	\hat{{\bm \beta}} = \argmin \limits_{{\bm \beta} \in \RR^d} \frac{1}{2} \|{\bm y}^{\rm obs} - X_{\rm inter} {\bm \beta}\|^2_2 + \lambda \|{\bm \beta}\|_1.
\end{align*}
Let us denote $\cA_{\rm inter} = \cA_{\rm inter}(\bm y^{\rm obs})$ be the active set of the  interaction model with $\bm y^{\rm obs}$, the conditional inference on the $j^{\rm th}$ selected feature in $\cA_{\rm inter}$ is defined as 
\begin{align} \label{eq:interaction_model_conditional_inference}
	\bm \eta_j^\top \bm Y \mid \{j \in \cA_{\rm inter}(\bm Y), \cA(\bm Y) = \cA(\bm y^{\rm obs}), \bm q(\bm Y) = \bm q (\bm y^{\rm obs}) \},
\end{align}
where $\eta_j = X_{\rm inter} (X_{\rm inter}^\top X_{\rm inter})^{-1} \bm e_j$ in which $\bm e_j \in \RR^d$.
We note that $\cA_{\rm inter}(\bm Y)$ is different from $\cA(\bm Y)$ which is the active set when we apply Lasso on data $\{X, \bm Y\}$.
By restricting the response vector to a line as in (\ref{eq:parametrized_data_space}), 
the conditional inference in (\ref{eq:interaction_model_conditional_inference}) is re-defined as 
\begin{align*}
		Z \mid \{Z \in \cZ\} \text{ where } \cZ = \{z \in \RR \mid j \in \cA_{\rm inter} (\bm y(z)), \cA(\bm y(z)) = \cA(\bm y^{\rm obs})\}.
\end{align*}
From now on, the process of identifying $\cZ$ is straightforward which is based on the method we proposed in \S3 of the main paper and the extension for full target case in the Appendix.

\subsection{Additional Experiments.} \label{appendix:exp_details}
For the experiments, we executed the code on Intel(R) Xeon(R) CPU E5-2687W v4 @ 3.00GHz.


\paragraph{Efficiency of the proposed method.}
We checked the computation time of our extension for elastic net when applying on synthetic data. 
The results are shown in Figure \ref{fig:no_polytopes_cc_elastic_net}.
\begin{figure}[H]
\centering
\includegraphics[width=0.45\linewidth]{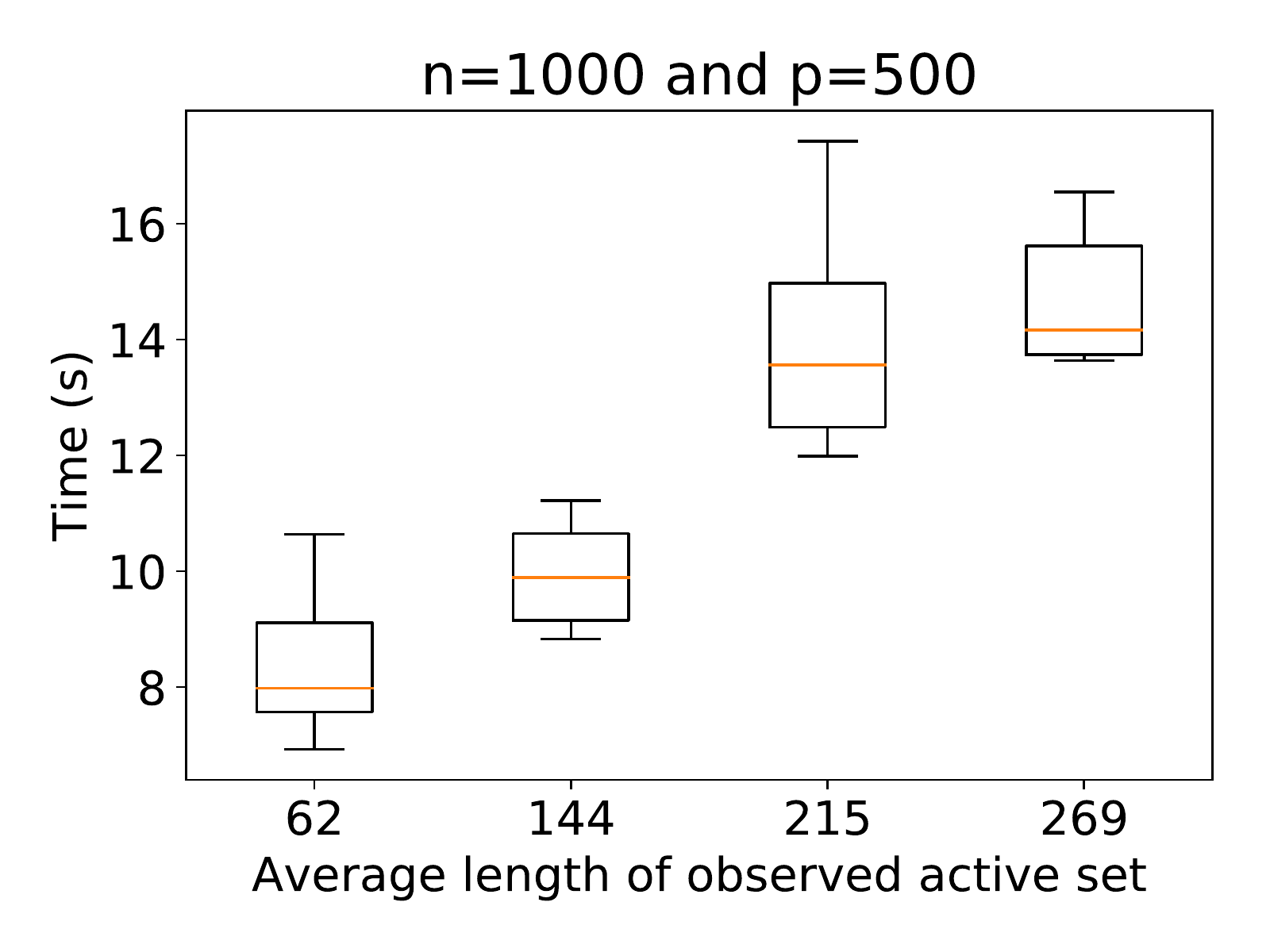}  
\caption{Computation time of our proposed method in elastic net case.}
\label{fig:no_polytopes_cc_elastic_net}
\end{figure}
%
%

\paragraph{The robustness of the proposed method in terms of the FPR control.}
We applied our proposed method to the case when the data follows Laplace distribution, skew normal distribution (skewness coefficient 10), and $t_{20}$ distribution. 
We also conducted experiments when $\sigma^2$ is also estimated from the data.
We generated $n$ outcomes as $y_i = \bm x_i^\top \bm \beta + \veps_i$, 
$i = 1, ..., n$, 
where 
$p = 5, \bm x_i \sim \NN(0, I_p)$, 
and $\veps_i$ follows Laplace distribution, skew normal distribution, or $t_{20}$ distribution with zero mean and standard deviation was set to 1.
In the case of estimated $\sigma^2$, $\veps_i \sim \NN(0, 1)$.
We set all elements of $\bm \beta$ to 0, and set $\lambda = 0.5$.
For each case, we ran 1,200 trials for each $n \in \{100, 200, 300, 400\}$. 
The FPR results are shown in Figure \ref{fig:fig_robustness}.

\begin{figure}[!t]
\begin{subfigure}{.5\textwidth}
  \centering
  \includegraphics[width=0.81\linewidth]{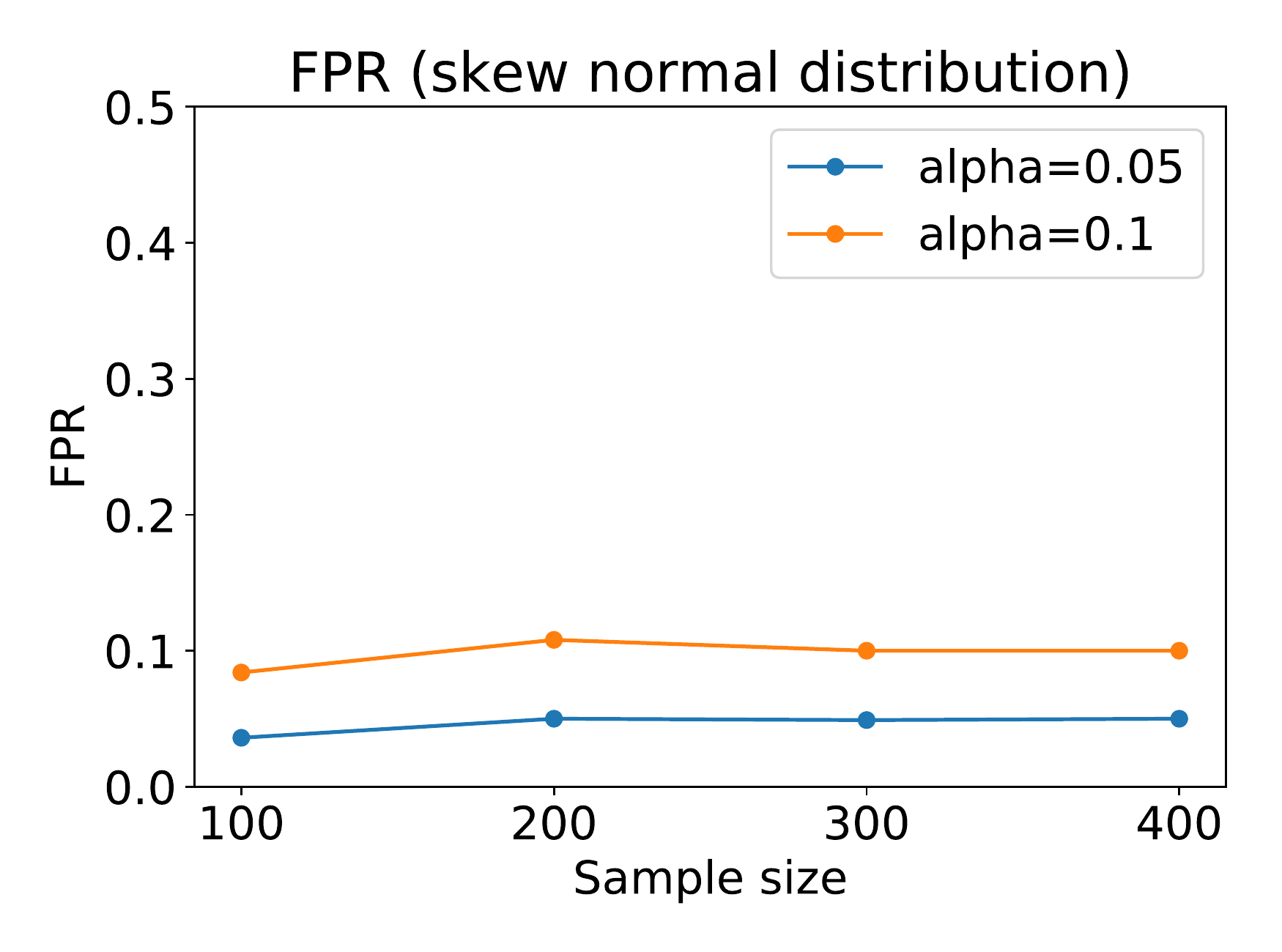}  
\end{subfigure}
\begin{subfigure}{.5\textwidth}
  \centering
  \includegraphics[width=.81\linewidth]{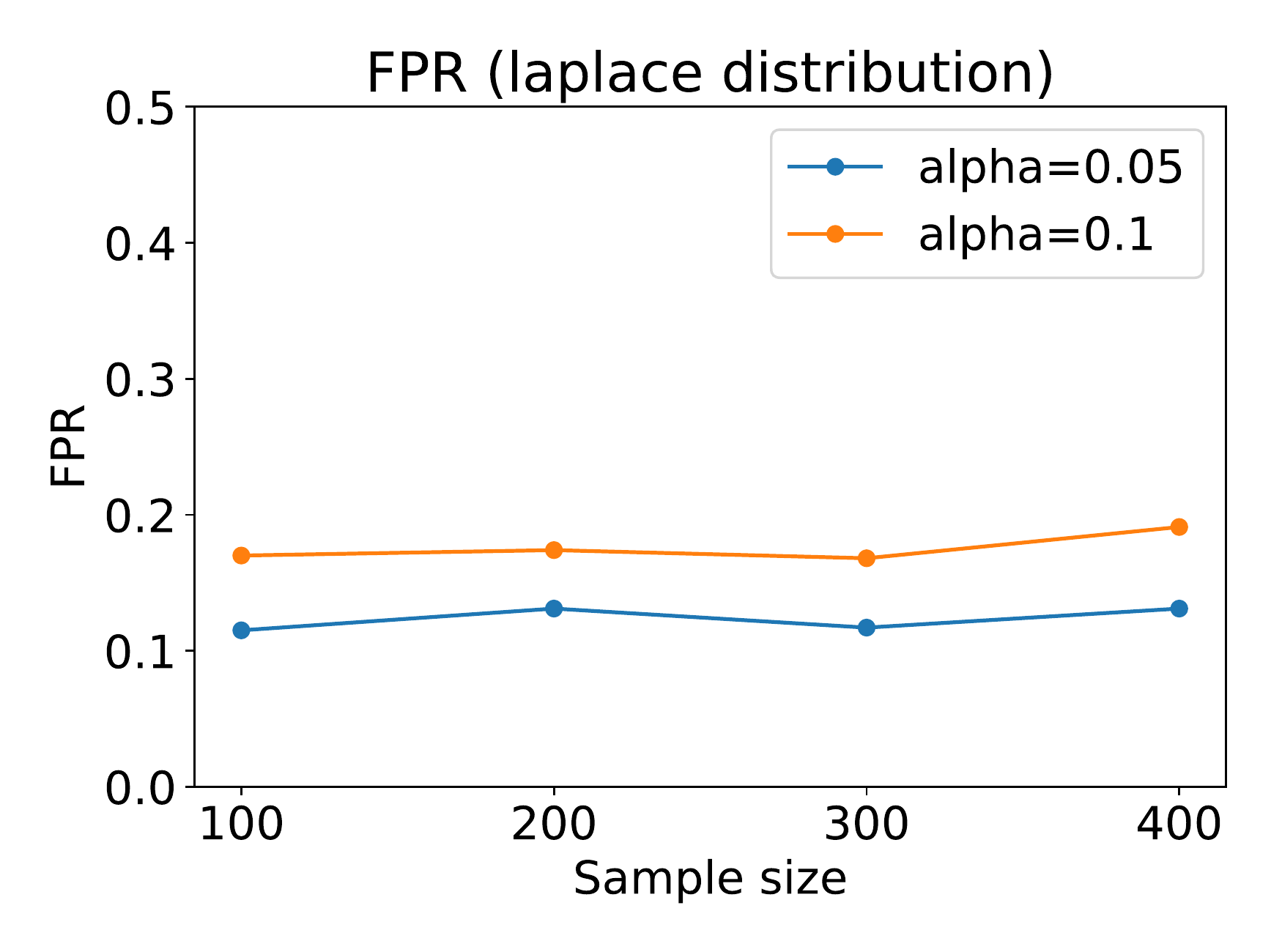}  
\end{subfigure}
\begin{subfigure}{.5\textwidth}
  \centering
  \includegraphics[width=.81\linewidth]{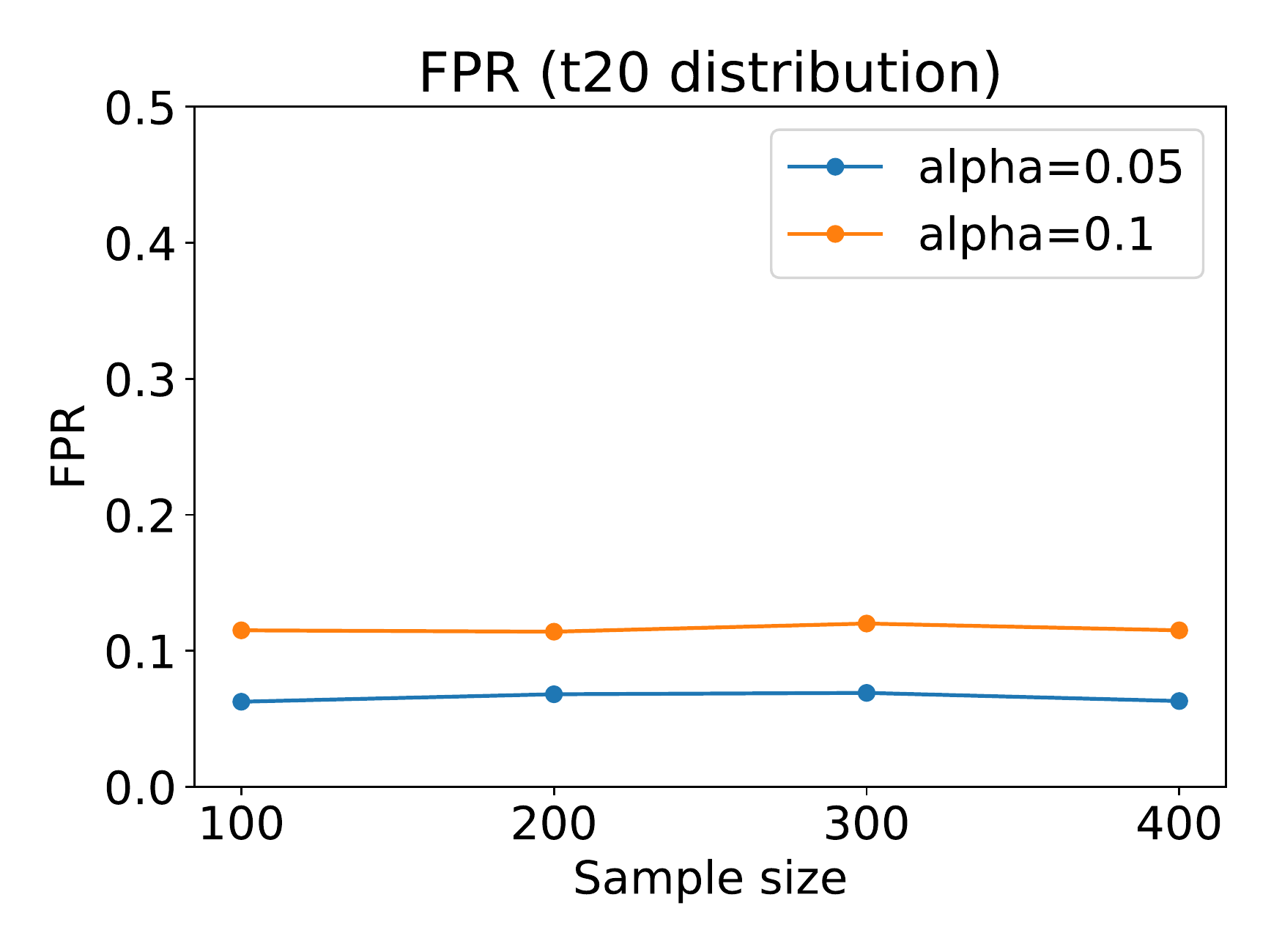}  
\end{subfigure}
\begin{subfigure}{.5\textwidth}
  \centering
  \includegraphics[width=.81\linewidth]{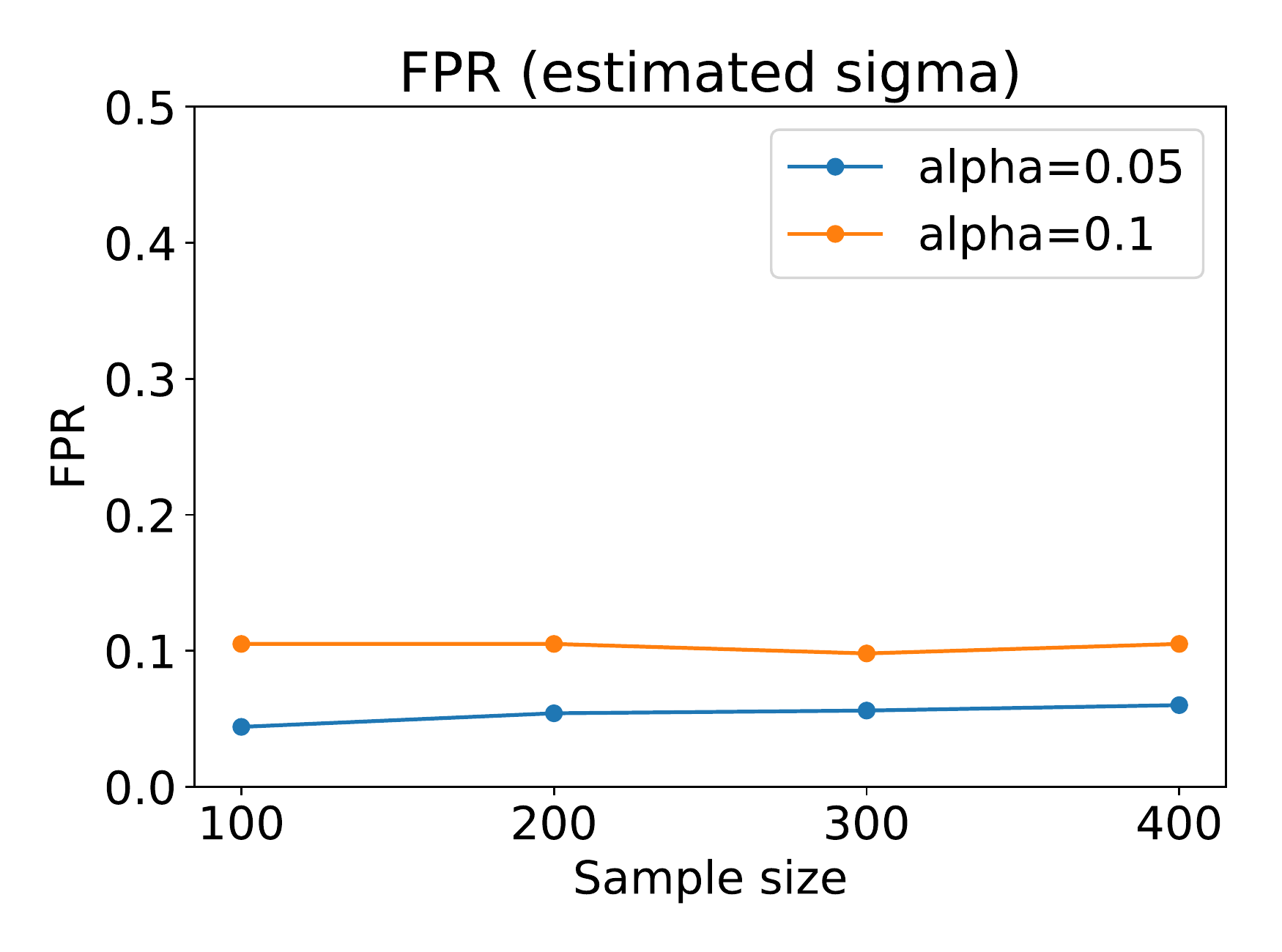}  
\end{subfigure}
\caption{The robustness of the proposed method in terms of the FPR control.}
\label{fig:fig_robustness}
\end{figure}


\begin{figure}[!t]
\begin{subfigure}{.245\textwidth}
  \centering
  \includegraphics[width=\linewidth]{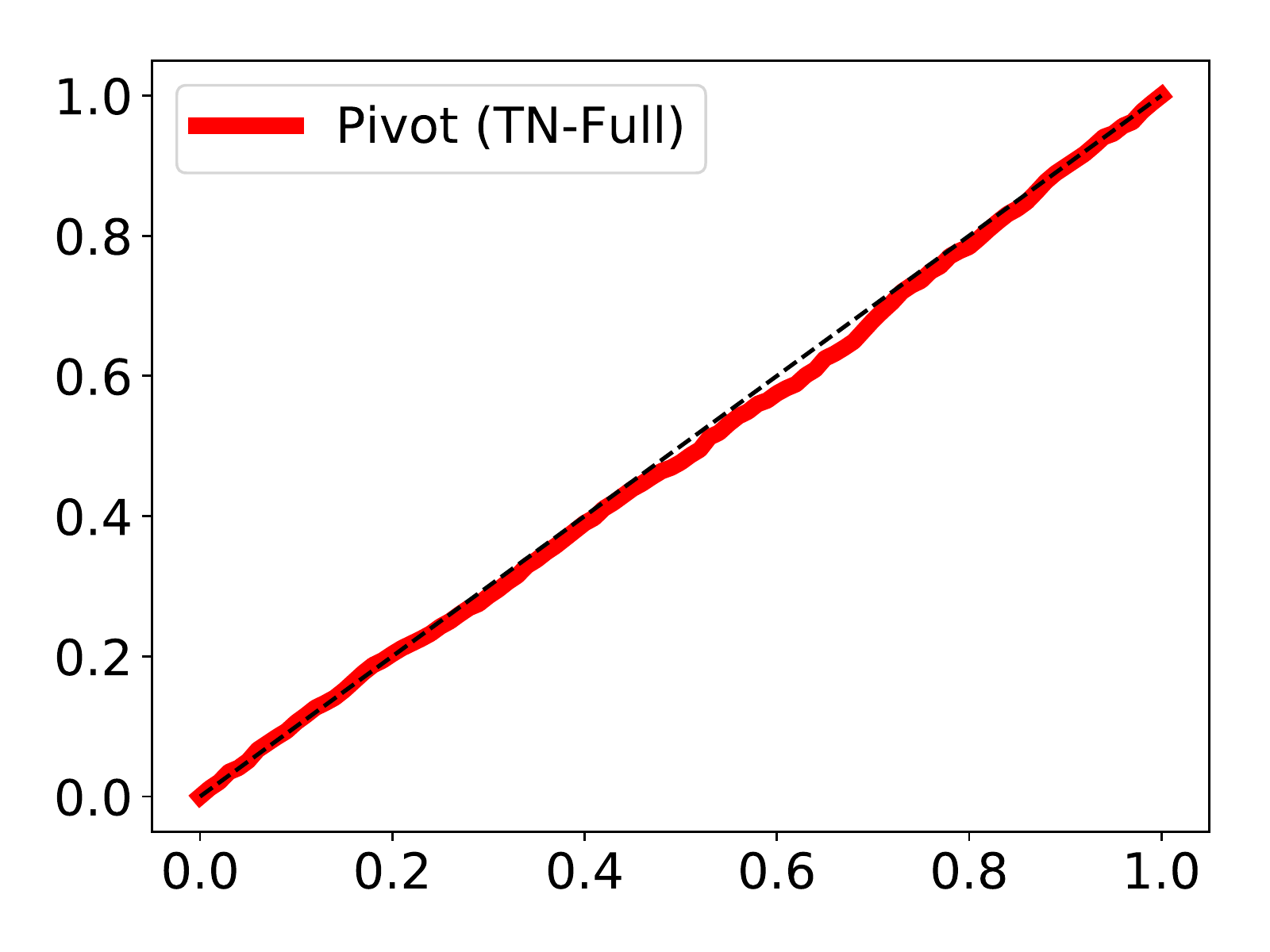}  
  \caption{TN-Full}
\end{subfigure}
\begin{subfigure}{.245\textwidth}
  \centering
  \includegraphics[width=\linewidth]{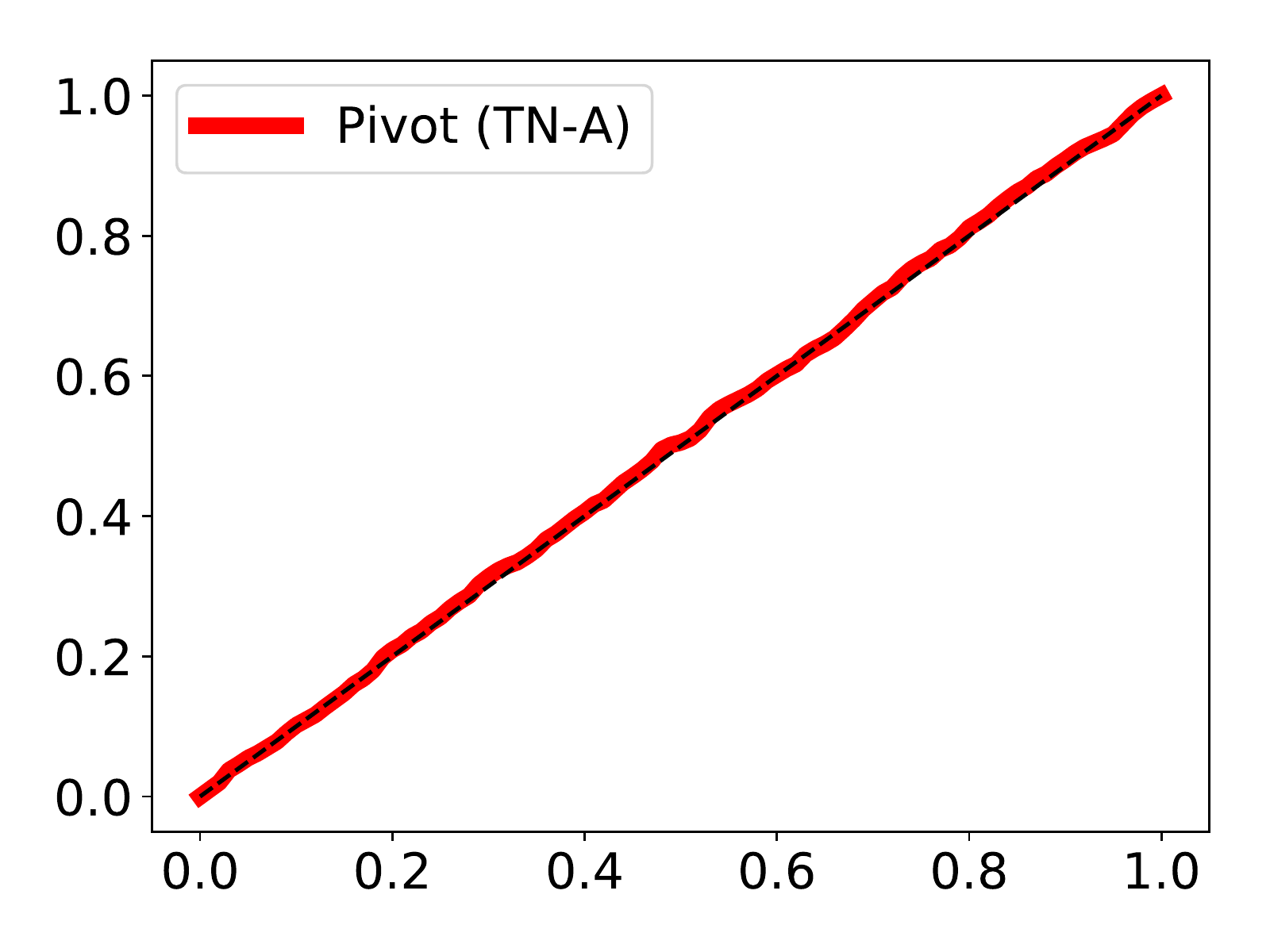}  
  \caption{TN-A}
\end{subfigure}
\begin{subfigure}{.245\textwidth}
  \centering
  \includegraphics[width=\linewidth]{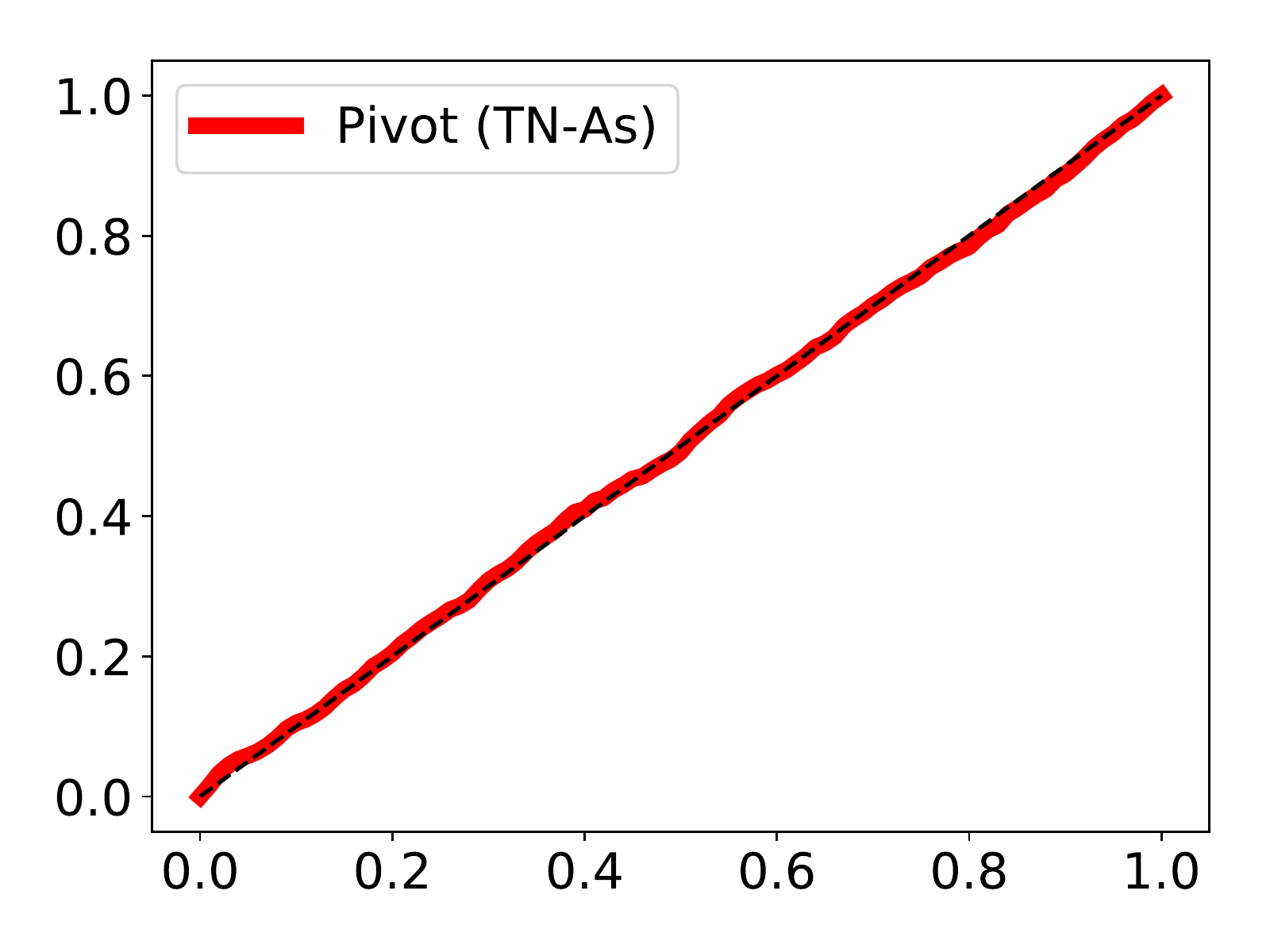}  
  \caption{TN-As}
\end{subfigure}
\begin{subfigure}{.245\textwidth}
  \centering
  \includegraphics[width=\linewidth]{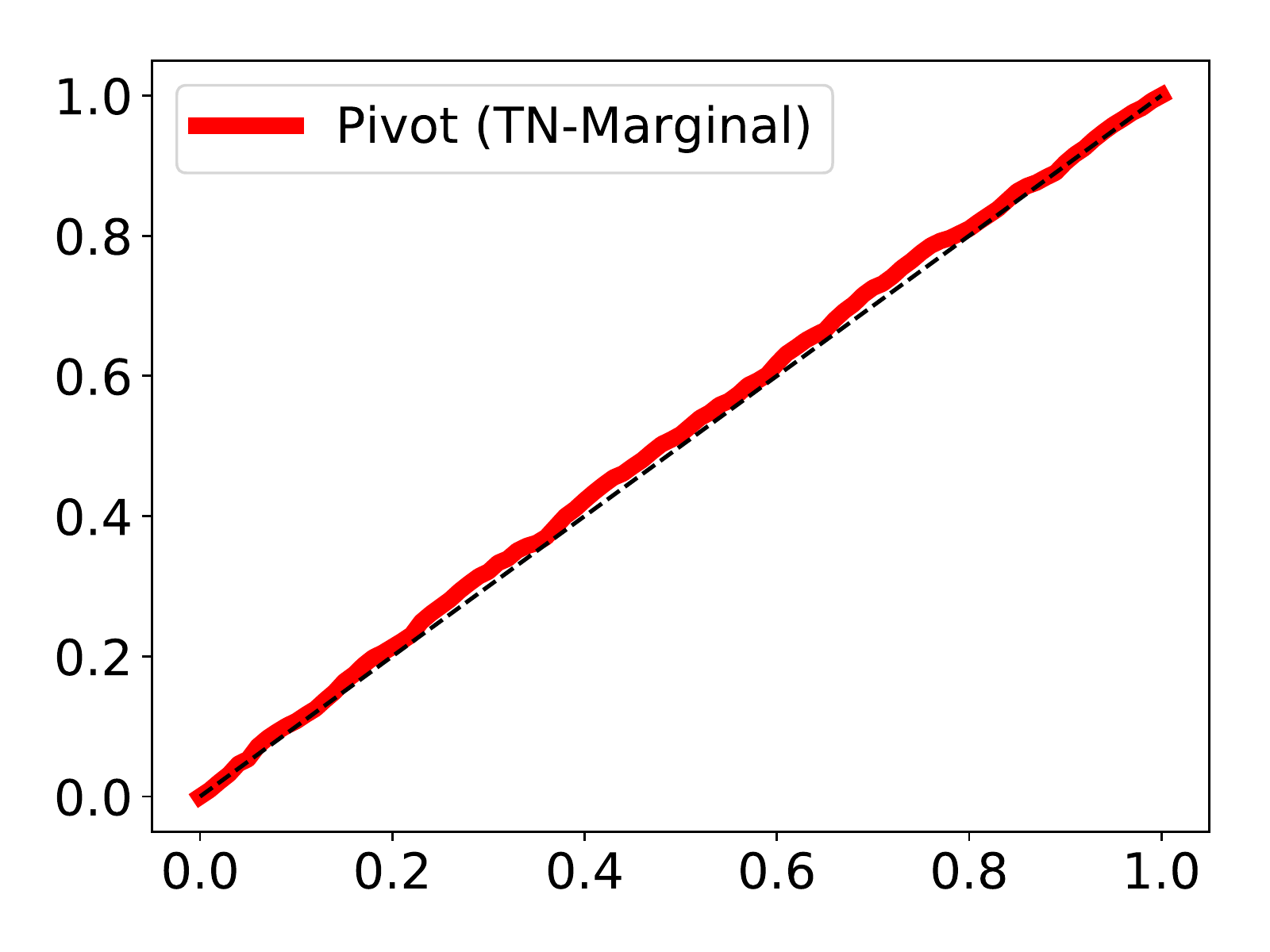}  
  \caption{TN-Marginal}
\end{subfigure}
\begin{subfigure}{.245\textwidth}
  \centering
  \includegraphics[width=\linewidth]{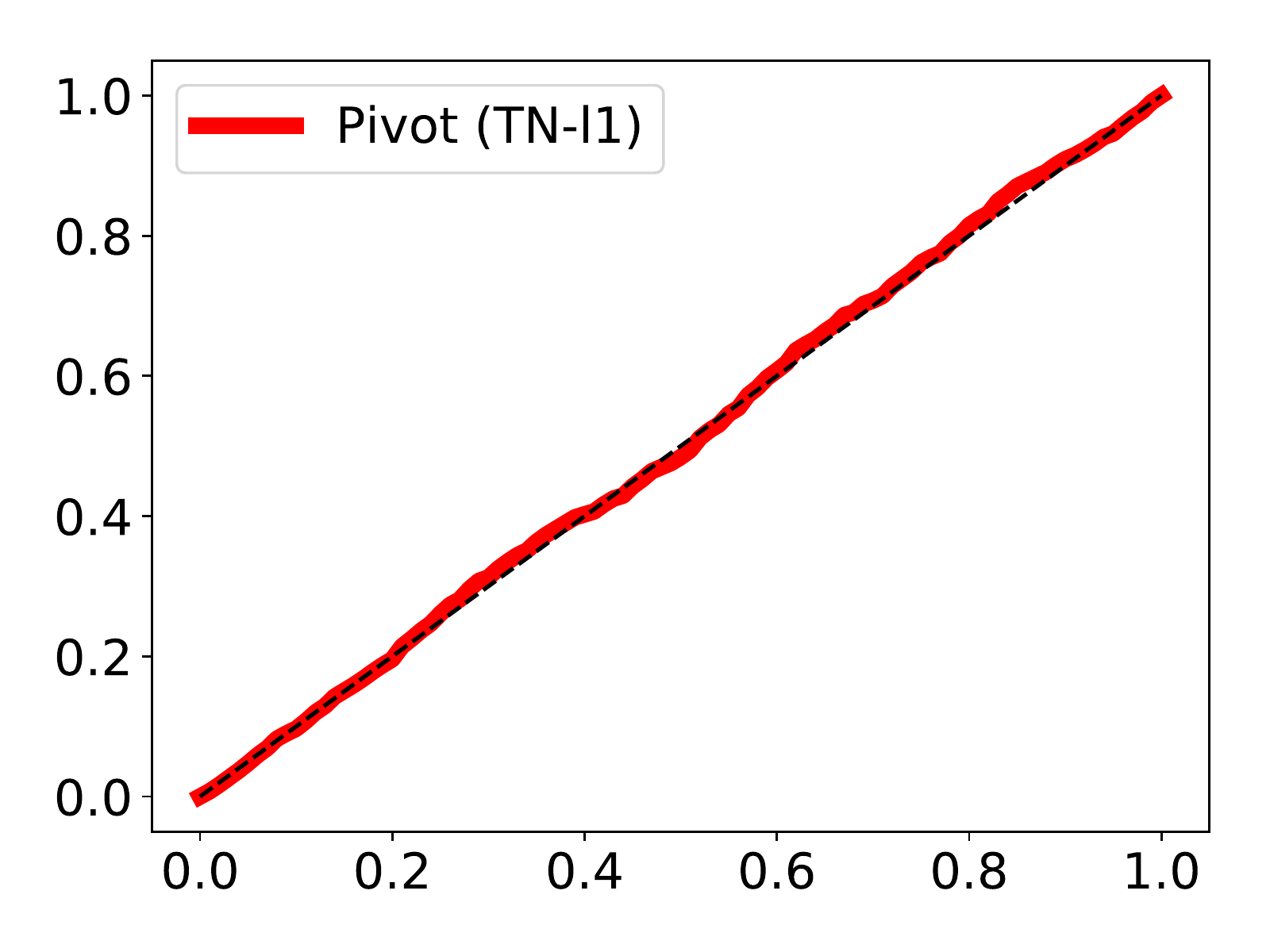}  
  \caption{TN-$\ell_1$}
\end{subfigure}
\begin{subfigure}{.245\textwidth}
  \centering
  \includegraphics[width=\linewidth]{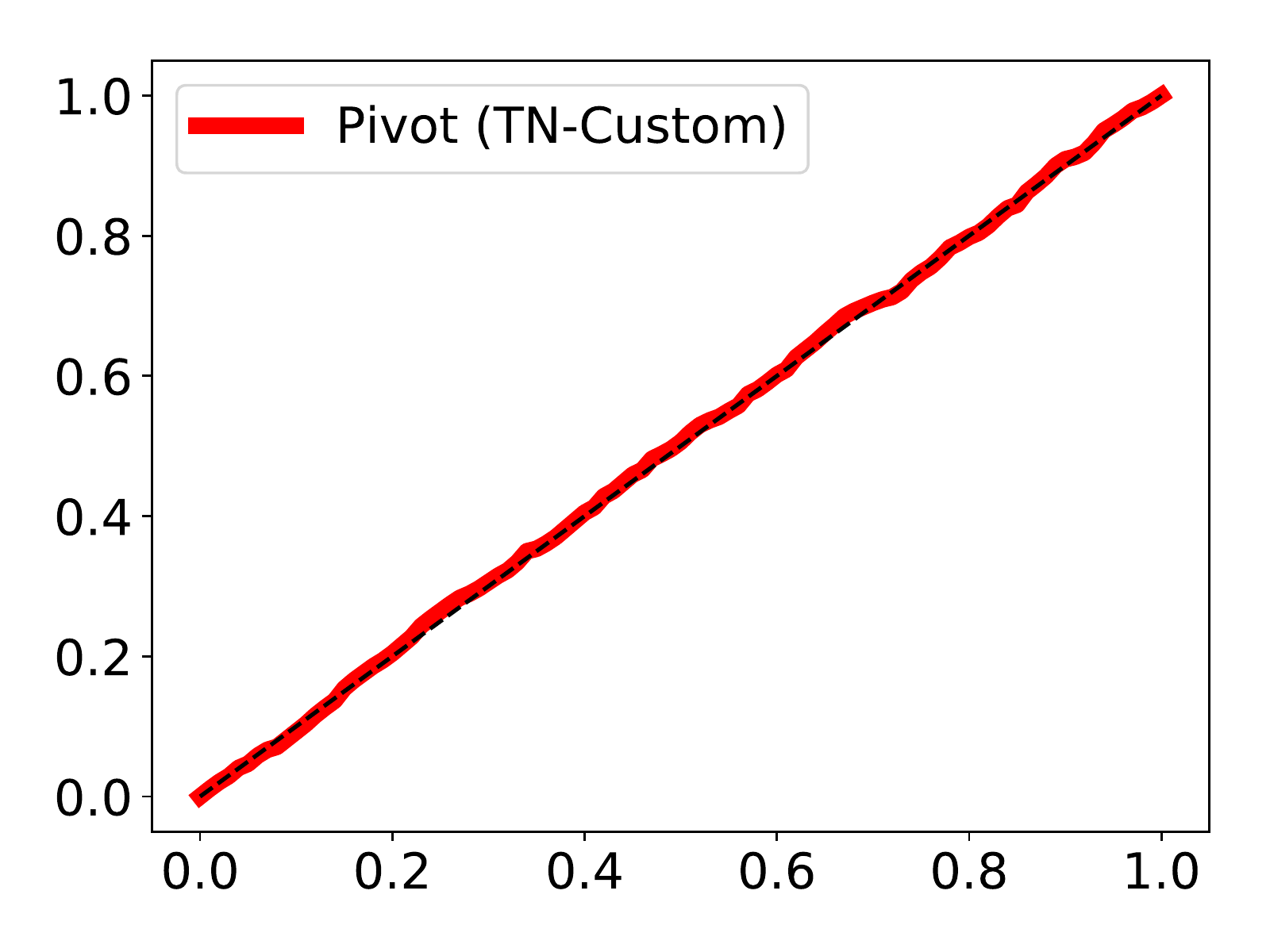}  
  \caption{TN-Custom}
\end{subfigure}
\begin{subfigure}{.245\textwidth}
  \centering
  \includegraphics[width=\linewidth]{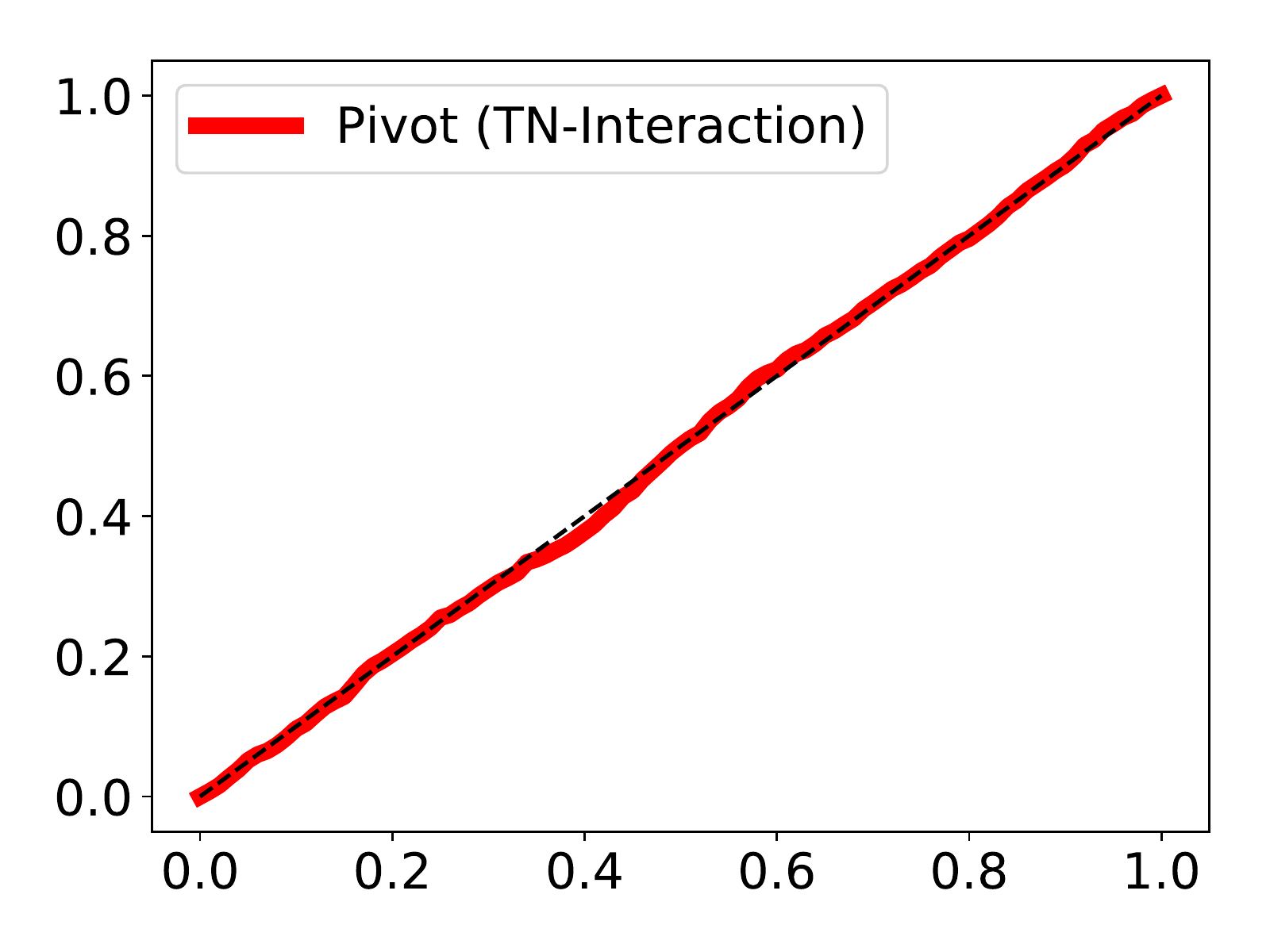}  
  \caption{TN-Interaction}
\end{subfigure}
\begin{subfigure}{.245\textwidth}
  \centering
  \includegraphics[width=\linewidth]{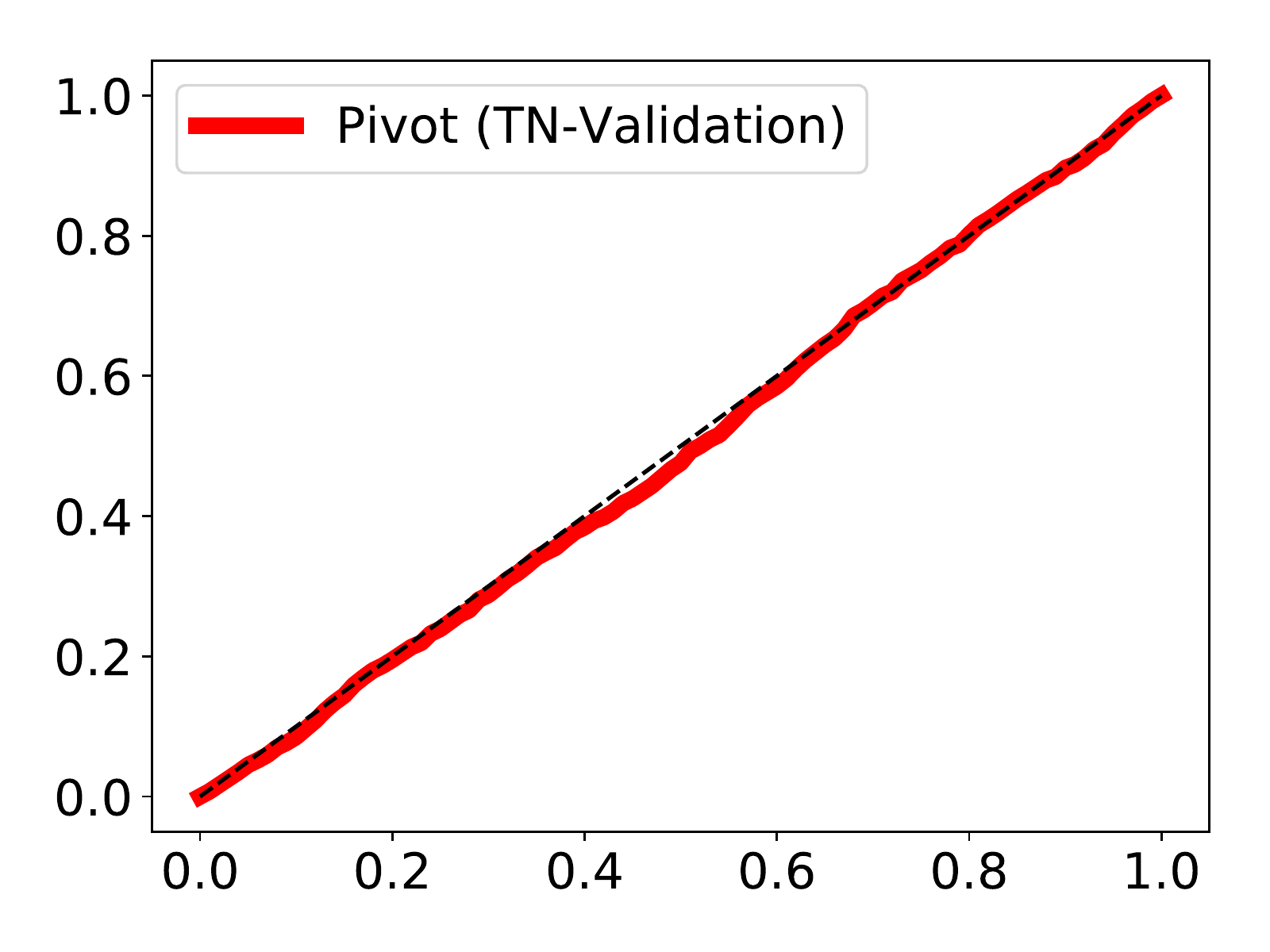}  
  \caption{TN-Validation}
\end{subfigure}
\caption{Uniform QQ-plot of the pivotal quantity.}
\label{fig:fig_uniform_qq_plot}
\end{figure}

\paragraph{Uniformity verification of the pivotal quantity.}
We generated $n = 100$ outcomes as $y_i = \bm x_i^\top \bm \beta + \veps_i$, 
$i = 1, ..., n$, 
where 
$p = 5, \bm x_i \sim \NN(0, I_p)$, 
and $\veps_i \sim \NN(0, 1)$.
We set the first two elements of $\bm \beta$ to 2, and set $\lambda = 5$.
We applied our method and ran 1,200 trials for each case of conditioning: TN-Full, TN-A, TN-As, TN-Marginal (marginal model), TN-$\ell_1$, TN-Custom, TN-Interaction (interaction model), and TN-Validation (considering validation selection event). 
For stable partial target formation, to identify $\cH_{\rm obs}$, we set the value of higher $\lambda$ to 15 in the case of TN-$\ell_1$, 
and cutoff value $c$ is set to 1 in the case of TN-Custom.
We set $\Lambda=\{2^{-1}, 2^0, 2^1\}$ and performed 5-fold cross-validation in the case of TN-Validation.
The results are shown in Figure \ref{fig:fig_uniform_qq_plot}.

\vfill

\end{document}